\documentclass[12pt,journal,draftcls,letterpaper,oneside,onecolumn]{IEEEtran}
\usepackage{algorithm}
\usepackage{algpseudocode}
\usepackage{bm}
\usepackage{amsmath}
\usepackage{amssymb} 
\usepackage{url} 
\usepackage[dvips,dvipdfm]{graphicx,color}
\usepackage{sidecap}
\usepackage{myDefinition_ogw}

\renewcommand{\algorithmicrequire}{\textbf{Input:}}
\renewcommand{\algorithmicensure}{\textbf{Output:}}

\newcommand{\libsvm}{LIBSVM}
\newcommand{\liblinear}{LIBLINEAR}

\begin{document}
\title{Safe Sample Screening for \\ Support Vector Machines}

\author{
Kohei~Ogawa,~\IEEEmembership{}
Yoshiki~Suzuki,~\IEEEmembership{}
%Yamato~Kawamoto,~\IEEEmembership{}
Shinya~Suzumura~\IEEEmembership{} 
and 
Ichiro~Takeuchi,~\IEEEmembership{Member,~IEEE}% <-this % stops a space
\IEEEcompsocitemizethanks{\IEEEcompsocthanksitem
K. Ogawa, Y. Suzuki,
%Y. Kawamoto,
S. Suzumura and I. Takeuchi
are with the Department of Engineering, Nagoya Institute of Technology,
Gokiso-cho, Showa-ku, Nagoya, Japan. \protect\\
E-mail: \{ogawa, suzuki,
%kawamoto,
suzumura\}.mllab.nit@gmail.com,
and
takeuchi.ichiro@nitech.ac.jp}% <-this % stops a space
\thanks{}}


\markboth{Ogawa, Suzuki, Suzumura and Takeuchi}%
{Shell \MakeLowercase{\textit{et al.}}: Bare Demo of IEEEtran.cls for Computer Society Journals}

%\IEEEpubid{\makebox[\columnwidth]{\hfill 0000--0000/00/\$00.00~\copyright~2007 IEEE}%

% PRIOR to the title within the \IEEEcompsoctitleabstractindextext IEEEtran
\IEEEcompsoctitleabstractindextext{%
\begin{abstract}
%\boldmath
Sparse classifiers such as the support vector machines (SVM) are efficient
in test-phases because the classifier is characterized only by a subset
of the samples called \emph{support vectors (SVs)},
and the rest of the samples (non SVs) have no influence on the
classification result.
However, 
the advantage of the sparsity has not been fully exploited 
in training phases
because it is generally difficult to know 
which sample turns out to be SV beforehand.
In this paper,
we introduce a new approach called
\emph{safe sample screening} 
that enables us to identify 
a subset of the non-SVs 
and screen them out 
prior to the training phase. 
Our approach is different from
existing heuristic approaches
in the sense that  
the screened samples are
\emph{guaranteed} to be non-SVs
at the optimal solution. 
We investigate the advantage of the safe sample screening approach
through intensive numerical experiments,
and demonstrate that
it can substantially decrease the computational cost
of the state-of-the-art SVM solvers such as {\libsvm}.
In the current \emph{big data} era, 
we believe that safe sample screening 
would be of great practical importance
since the data size can be reduced without sacrificing the optimality of the final solution.

\end{abstract}
% not want either math or citations to appear in the abstract.

% Note that keywords are not normally used for peer review papers.
}

\maketitle

% is not selected <OR> if conference mode is selected - because compsoc
% conference papers position the abstract like regular (non-compsoc)
\IEEEdisplaynotcompsoctitleabstractindextext

% \fi
\IEEEpeerreviewmaketitle

\section{Introduction}
\label{sec:Introduction}
\IEEEPARstart{T}he support vector machines (SVM) 
\cite{Boser92a,Cortes95a,Vapnik98a}
has been successfully applied to 
large-scale classification problems
\cite{Ma09a,Lin11a,Lin11b}.
A trained SVM classifier is
\emph{sparse}
in the sense that
the decision function is characterized only by 
a subset of the samples 
known as
\emph{support vectors (SVs)}.
One of the computational advantages of such a sparse classifier is
its efficiency
in the test phase,
where 
the classifier can be evaluated for a new test input
with the cost proportional only to the number of the SVs.
The rest of the samples (non-SVs) can be discarded \emph{after} training phases
because they have no influence on the classification results.

However, 
the advantage of the sparsity has not been fully exploited 
in the training phase 
because it is generally difficult to  know
which sample turns out to be SV beforehand.
Many existing SVM solvers
spend most of their time 
for identifying the SVs
\cite{Platt99a,Joachims99b,Hastie04a,Scheinberg06a,Chang11a}.
For example,
well-known 
{\libsvm}
\cite{Chang11a} 
first predicts which sample would be SV (prediction step),
and then 
solves a smaller optimization problem
defined only with the subset of the samples
predicted as SVs (optimization step).
These two steps must be repeated until the true SVs are identified 
because some of the samples might be mistakenly predicted as non-SVs in the prediction step. 

In this paper,
we introduce a new approach 
that can identify 
a subset of the non-SVs
and screen them out 
\emph{before}
actually solving the training optimization problem.
Our approach is different from
the prediction step in the above
{\libsvm}
or
other similar heuristic approaches
in the sense that
the screened samples
are
\emph{guaranteed} to be non-SVs
at the optimal solution. 
It means that the original optimal solution can be obtained
by solving the smaller problem defined
only with the remaining set of the non-screened samples. 
We call our approach as
\emph{safe sample screening}
because it never identifies a true SV as non-SV.
\figurename \ref{fig:toy.example}
illustrates our approach 
on a toy data set
(see \S \ref{subsec:toy.experiment} for details).

% --------------------------------------------------
% Fig1
% --------------------------------------------------
\begin{figure}[t]
\centering
\includegraphics[width = 0.35\textwidth]{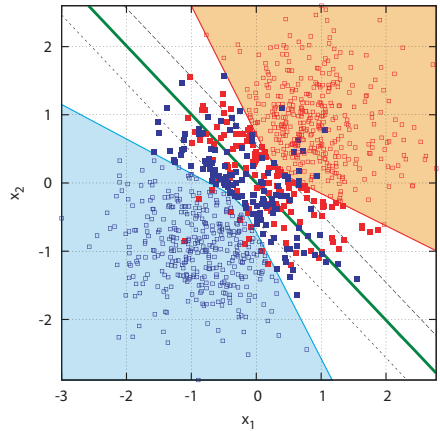}
\caption{An example of our safe sample screening method on a binary
 classification problem with a two-dimensional toy data set.
 For each of the red and blue classes, 500 samples are drawn.
 Our safe sample screening method found that all the samples in
 the shaded regions
% (red/blue samples in red/blue shaded regions,respectively)
 are guaranteed to be non-SVs.
 In this example, more than 80\% of the samples
 (\red{\scriptsize $\square$} and \blue{\scriptsize $\square$}) 
 are 
 identified as non-SVs and they can be discarded prior to the training phase.
 It means that 
 the optimal classifier (the green line)
 can be obtained by solving a much smaller optimization problem
 defined only with the remaining 20\% of the samples
 (\red{\scriptsize $\blacksquare$} and \blue{\scriptsize $\blacksquare$}).
 See \S \ref{subsec:toy.experiment} for details.
 }
\label{fig:toy.example}
\end{figure}
% ---------------------------------------------------------------------------

Safe sample screening can be used together with any SVM solvers such as
{\libsvm} 
as a preprocessing step for reducing the training set size.
%
% Safe sample screening can be used for reducing the data set size as a
% preprocessing step of any SVM solvers such as {\libsvm}.
%
In our experience,
it is often possible to screen out nearly 90\% of the samples as non-SVs.
In such cases,
the total computational cost of SVM training can be substantially reduced because 
only the remaining 10\% of the samples are fed into an SVM solver
(see \S \ref{sec:experiments}).
Furthermore,
we show that
safe sample screening
is especially useful
for model selection,
where a sequence of SVM classifiers with different regularization parameters are trained.
In the current \emph{big data} era, 
we believe that safe sample screening 
would be of great practical importance
because it enables us to reduce 
the data size
without sacrificing the optimality.

The basic idea behind safe sample screening 
is inspired by a resent study by El Ghaoui et al.
\cite{ElGhaoui12b}.
In the context of $L_1$ regularized sparse linear models,
they introduced an approach 
that can
\emph{safely}
identify
a subset of the
\emph{non-active features}
whose coefficients turn out to be zero at the optimal solution. 
This approach has been called
\emph{safe feature screening},
and various extensions have been reported 
\cite{Xiang12a,Xiang12b,Dai12a,Wang12a,Wang13c,Wang13d,Wu13a,Wang13a,Wang13b}
(see \S \ref{subsec:relation.with.feature.screening} for details).
Our contribution is to extend the idea of 
\cite{ElGhaoui12b}
for safely screening out non-SVs.
This extension is non-trivial
because the feature sparseness in a linear model 
stems from the $L_1$ penalty,
while the sample sparseness in an SVM
is originated from the large-margin principle.

This paper is an extended version of our preliminary conference paper
\cite{Ogawa13a}, 
where
we proposed a safe sample screening method 
that can be used
in somewhat more restricted situation 
than we consider here
(see Appendix
\ref{app:v.s.ICML.ver.}
for details). 
In this paper,
we extend our previous method 
in order to overcome the limitation and to improve the screening performance. 
%
%Furthermore,
%we empirically compared our old and new approaches in appendix
%\ref{app:v.s.ICML.ver.},
%where we observed that the new one has better performances.
%
As the best of our knowledge, 
our approach in 
\cite{Ogawa13a}
is the first safe sample screening method. 
After our conference paper was published,
Wang et al.
\cite{Wang13e}
recently proposed a new method and demonstrated that it performed better than our previous method in 
\cite{Ogawa13a}.
In this paper,
we further go beyond the Wang et al.'s method,
and show that our new method has better screening performance 
from both theoretical and empirical viewpoints
(see \S \ref{subsec:relation.with.feature.screening} for details).

The rest of the paper is organized as follows.
In \S \ref{sec:SVM},
we formulate the SVM and summarize the optimality conditions. 
Our main contribution is presented in 
\S \ref{sec:Safe.Sample.Screening.Rule} 
where we propose three safe sample screening methods for SVMs.
In \S \ref{sec:in.practice},
we describe how to use the proposed safe sample screening methods in practice.
Intensive experiments
are conducted
in \S \ref{sec:experiments},
where we investigate how much the computational cost of the state-of-the-art SVM solvers can be reduced
by using safe sample screening.
We summarize our contribution and future works 
in \S \ref{sec:Conclusion}.
Appendix contains the proofs of all the theorems and the lemmas, 
a brief description of (and comparison with) our previous method in our preliminary conference paper \cite{Ogawa13a},
the relationship between our methods and the method in \cite{Wang13e}, 
and some deitaled experimental protocols. 
The C++ and Matlab codes are available at
{\it http://www-als.ics.nitech.ac.jp/code/index.php?safe-sample-screening}.

\noindent
{\bf Notation}:
We let
$\bR$,
$\bR_+$
and
$\bR_{++}$
be the set of
real,
nonnegative 
and 
positive numbers,
respectively. 
We define 
$\bN_n \triangleq \{1, \ldots, n\}$
for any natural number $n$.
Vectors and matrices are represented
by bold face lower and upper case characters
such as
$\bm v \in \RR^n$
and 
$\bm M \in \RR^{m \times n}$,
respectively. 
An element of a vector 
$\bm v$
is written as
$v_i$
or
$(\bm v)_i$.
Similarly, 
an element of a matrix 
$\bm M$
is written as 
$M_{ij}$
or 
$(\bm M)_{ij}$.
Inequalities between two vectors such as 
$\bm v \le \bm w$
indicate component-wise inequalities:
$v_i \le w_i ~ \forall i \in \NN_n$.
Unless otherwise stated, 
we use 
$\|\cdot\|$
as a Euclidean norm. 
A vector of all 0 and 1 are denoted as 
$\zeros$ 
and 
$\ones$, 
respectively.

\section{Support vector machine}
\label{sec:SVM}
In this section we formulate the support vector machine (SVM). 
Let us consider a binary classification problem with
$n$ samples and $d$ features. 
We denote the training set as 
$\{(\bm x_i,y_i)\}_{i\in\bN_n}$
where 
$\bm x_i \in \cX \subseteq \bR^d$
and 
$y_i \in \{-1, +1\}$.
%
%For our proposed method, 
We consider a linear model in a feature space
$\cF$
in the following form: 
\begin{eqnarray*}
 f(\bm x) = \bm w^\top \Phi(\bm x_i),
\end{eqnarray*}
where 
$\Phi: \cX \rightarrow \cF$ 
is a map from the input space 
$\cX$ 
to the feature space 
$\cF$, 
and 
$\bm w \in \cF$ 
is a vector of the coefficients\footnote{
The bias term can be augmented to
$\bm w$ and $\Phi(\bm x)$
as an additional dimension.
}.
We sometimes write 
$f(\bm{x})$
as
$f(\bm x; \bm w)$
for explicitly specifying the associated parameter $\bm w$. 
The optimal parameter 
$\bm w^*$ 
is obtained by solving 
\begin{eqnarray}
\label{eq:p.SVM.prob.}
\bm w^*
\triangleq
\arg \min_{\bm w \in \cF}
~
\frac{1}{2} \| \bm w \|^2
+ C \sum_{i \in \bN_n} \max\{0, 1 - y_i f(\bm x_i)\},
\end{eqnarray}
where 
$C \in \RR_{++}$ 
is the regularization parameter.
The loss function 
$\max\{0, 1 - y_i f(\bm x_i)\}$
is known as {\it hinge-loss}.
We use a notation such as 
$\bm w_{[C]}^*$
when we emphasize that
it is the optimal solution of the problem
\eq{eq:p.SVM.prob.}
associated with the regularization parameter $C$.

The dual problem of
\eq{eq:p.SVM.prob.}
is formulated with the Lagrange multipliers
$\bm \alpha \in \RR_+^n$
as
\begin{eqnarray}
\label{eq:d.SVM.prob.}
\bm \alpha_{[C]}^* 
\triangleq
\arg
\max_{\bm \alpha}
~
\bigl(
\cD(\bm \alpha)
\triangleq
-\frac{1}{2} \sum_{i,j \in \bN_n} 
\alpha_i \alpha_j Q_{ij}
+ 
\sum_{i \in \NN_n} \alpha_i
\bigr)
~~~
\text{s.t.}
~
0 \le \alpha_i \le C, i \in \NN_n,
\end{eqnarray}
where 
$\bm Q \in \bR^{n \times n}$ 
is an
$n \times n$ 
matrix 
defined as 
$Q_{ij} \triangleq y_i y_j K(\bm x_i, \bm x_j)$ 
and
$K(\bm x_i, \bm x_j) \triangleq \Phi(\bm x_i)^\top \Phi(\bm x_j)$
is the \emph{Mercer kernel function}
defined by the feature map $\Phi$. 

%From the KKT condition 
%\eq{eq:SVM.KKT.w}, 
Using the dual variables, 
the model 
$f$ 
is written
%with 
%$\bm \alpha$
as
%follows:
%
\begin{eqnarray}
\label{eq:d.model}
f(\bm x) = \sum_{i\in\bN_n} \alpha_i y_i K(\bm x_i, \bm x).
\end{eqnarray}

%By using 
%$\bm \alpha \ge \zeros, \bm \eta \ge \zeros$ 
%and 
%the complementarity conditions, 
Denoting the optimal dual variables as
$\{\alpha^*_{[C]i}\}_{i \in \NN_n}$,
the optimality conditions
of the SVM
%\eq{eq:d.SVM.prob.} 
are summarized as 
\begin{eqnarray}
\label{eq:SVM.opt.}
i \in \cR \Rightarrow \alpha_{[C]i}^* = 0, 
~~~
i \in \cE \Rightarrow \alpha_{[C]i}^* \in [0,C], 
~~~
i \in \cL \Rightarrow \alpha_{[C]i}^* = C, 
\end{eqnarray}
where we define the three index sets:
\begin{eqnarray*}
\cR \triangleq \{ i \in \bN_n~|~ y_i f(\bm x_i) > 1 \}, 
~~~
\cE \triangleq \{ i \in \bN_n~|~ y_i f(\bm x_i) = 1 \}, 
~~~
\cL \triangleq \{ i \in \bN_n~|~ y_i f(\bm x_i) < 1 \}.  
\end{eqnarray*}

The optimality conditions
\eq{eq:SVM.opt.}
suggest that,
if it is known {\it a priori}
which samples turn out to be the members of $\cR$ at the optimal solution, 
those samples can be discarded before actually solving the training optimization problem
because the corresponding $\alpha^*_{[C]i} = 0$ indicates that they have no influence on the solution.
Similarly,
if some of the samples are known {\it a priori} to be the members of $\cL$
at the optimal solution, 
the corresponding variable can be fixed as $\alpha^*_{[C]i} = C$.
If 
we let 
$\cR^\prime$
and
$\cL^\prime$
be the subset of the samples known as the members of 
$\cR$
and 
$\cL$,
respectively, 
one could first compute
$d_i \triangleq C \sum_{j \in \cL^\prime} y_j K(\bm x_i, \bm x_j)$
for all
$i \in \NN_n \setminus (\cR^\prime \cup \cL^\prime)$,
and put them in a cache.
Then, 
it is suffice to solve the following smaller optimization problem defined only with
the remaining subset of the samples 
and the cached variables\footnote{
Note that the samples in
$\cL^\prime$
are needed in the future test phase.
Here, 
we only mentioned that
the samples in
$\cR^\prime$
and 
$\cL^\prime$
are not used during the training phase. 
}:
\begin{eqnarray*}
\max_{\bm \alpha}
\sum_{i,j \in \bN_n \setminus (\cR^\prime \cup \cL^\prime)} \!\!\!\!\! \alpha_i \alpha_j Q_{ij}
- \sum_{i \in \bN_n \setminus (\cR^\prime \cup \cL^\prime) } \!\!\!\!\! \alpha_i (1 - d_i)
~~~{\rm s.t.}~ 0 \le \alpha_i \le C,~i \in \bN_n \setminus (\cR^\prime \cup \cL^\prime). 
\end{eqnarray*}
%in the training phase
%where all the computation can be carried out 
%without the samples in $\cR^\prime \cup \cL^\prime$

%$f(\bm x_i) = \sum_{j \in \NN_n \setminus (\cL^\prime \cup \cR^\prime)} \alpha_i + C g_i$.
%
%Namely,
%if there are some prior knowledge on the three index sets
%$\cR, \cE, \cL$
%at the optimal solution, 
%the training phase can be extremely simplified. 
%
%In fact,
%many existing SVM solvers spend most of their time 
%for allocating each training sample to one of these three index 
%sets~\cite{Joachims99b,Platt99a,Cauwenberghs01,HasRosTibZhu04,Fan08b,Chang11a}.

Hereafter, 
the training samples in $\cE$ are called \emph{support vectors (SVs)},
while those in 
$\cR$ and $\cL$
are called
\emph{non-support vectors (non-SVs)}.
Note that
\emph{support vectors}
usually indicate the samples both in $\cE$ and $\cL$
in the machine learning literature
(we also use the term SVs in this sense in the previous section).
We adopt the above uncommon terminology
because the samples in
$\cR$ and $\cL$
can be treated almost in an equal manner in the rest of this paper. 
In the next section,
we develop three types of testing procedures for screening out a subset of the non-SVs. 
Each of these tests are conducted by evaluating a simple rule for each sample.
We call these testing procedures as 
\emph{safe sample screening tests} 
and
the associated rules as 
\emph{safe sample screening rules}.
%
%We call them
%``safe''
%because 
%the screened samples are \emph{guaranteed} to be non-SVs
%at the optimal solution,
%i.e.,
%there is absolutely no risk of falsely detecting a SV as non-SV.
%
%It means that the solution of a smaller optimization problem defined
%only with the remaining samples is \emph{guaranteed} to be
%optimal.

\section{Safe Sample Screening for SVMs}
\label{sec:Safe.Sample.Screening.Rule}
In this section, we present our safe sample screening approach for SVMs.
\subsection{Basic idea}
\label{subsec:Basic.idea}
Let us consider a situation that we have a region 
$\Theta_{[C]} \subset \cF$ 
in the solution space,
where we only know that
the optimal solution
$\bm w^*_{[C]}$
is somewhere in this region
$\Theta_{[C]}$,
%$\bm w^*_{[C]} \in \Theta_{[C]}$, 
%i.e.,
%the optimal solution is somewhere in this region, 
but 
$\bm w^*_{[C]}$
itself is unknown. 
In this case, 
the optimality conditions
\eq{eq:SVM.opt.}
indicate that 
\begin{eqnarray}
\label{eq:basic.test.l}
&&
\bm{w}^*_{[C]} \in \Theta_{[C]}
~\wedge~
\min_{\bm{w} \in \Theta_{[C]}} y_i f(\bm{x}_i ; \bm{w}) > 1
~ \Rightarrow ~ 
y_i f(\bm{x}_i ; \bm{w}^*_{[C]}) > 1
~ \Rightarrow ~
\alpha_{[C]i}^* = 0.
\\
\label{eq:basic.test.u}
&&
\bm{w}^*_{[C]} \in \Theta_{[C]}
~\wedge~
\max_{\bm{w} \in \Theta_{[C]}} y_i f(\bm{x}_i ; \bm{w}) < 1
~ \Rightarrow ~ 
y_i f(\bm{x}_i ; \bm{w}^*_{[C]}) < 1
~ \Rightarrow ~
\alpha_{[C]i}^* = C.
\end{eqnarray}
These facts imply that,
even if the optimal 
$\bm w^*_{[C]}$
itself is unknown, 
we might have a chance to screen out a subset of the samples in $\cR$ or $\cL$.
%if 
%such a 
%$\Theta_{[C]}$
%is available.

Based on the above idea,
we construct safe sample screening rules in the following way:
\begin{description}
 \item [(Step 1)]~~~
     we construct a region
      $\Theta_{[C]}$
      such that 
      \begin{eqnarray}
       \label{eq:basic.theta}
       \bm w^*_{[C]} \in \Theta_{[C]} \subset \cF. 
      \end{eqnarray}
 \item [(Step 2)]~~~
       we compute the lower and the upper bounds:
       \begin{eqnarray}
        \label{eq:basic.lower.upper.bounds}
	 \ell_{[C]i}
	 \triangleq
	 \min_{\bm{w} \in \Theta_{[C]}} y_i f(\bm{x}_i ; \bm{w}), 
	 ~~~
	 u_{[C]i}
	 \triangleq
	 \max_{\bm{w} \in \Theta_{[C]}} y_i f(\bm{x}_i ; \bm{w})
	 ~~~
	 \forall i \in \bN_n.
       \end{eqnarray}	    
\end{description}
Then,
the safe sample screening rules are written as 
  \begin{eqnarray}
 \label{eq:basic.safe.sample.screening.rule}
    \ell_{[C]i} > 1
    ~\Rightarrow~
    i \in \cR
    ~\Rightarrow~
    \alpha^*_{[C]i} = 0,
    ~~~
    u_{[C]i} < 1
    ~\Rightarrow~
    i \in \cL
    ~\Rightarrow~
    \alpha^*_{[C]i} = C.
  \end{eqnarray}
%
%Remembering that
%these rules should be able to be evaluated
%much more efficiently
%than solving the optimization problem itself, 
%we must develop efficient algorithms for
%computing the lower and upper bounds
%in 

In section \ref{subsec:Ball.Test},
we first study so-called
\emph{Ball Test}
where the region
$\Theta_{[C]}$
is a closed ball in the solution space. 
In this case, 
the lower and the upper bounds can be obtained in closed forms.
In section \ref{subsec:Ball.tests.for.SVM},
we describe how to construct such a ball
$\Theta_{[C]}$
for SVMs, 
and introduce two types of balls
$\Theta_{[C]}^{\rm (BT1)}$
and 
$\Theta_{[C]}^{\rm (BT2)}$.
We call the corresponding tests as
\emph{Ball Test 1 (BT1)}
and
\emph{Ball Test 2 (BT2)},
respectively.
In section 
\ref{subsec:Intersection.Test},
we combine these two balls and develop so-called
\emph{Intersection Test (IT)},
which is shown to be more powerful (more samples can be screened out) than BT1 and BT2.

%
% sec3.2
%
\subsection{Ball Test}
\label{subsec:Ball.Test}
When
$\Theta_{[C]}$
is a closed ball,
the lower or the upper bounds of
$y_i f(\bm x_i)$
can be obtained by minimizing a linear objective subject to a single quadratic constraint. 
We can easily show that  the solution of this class of optimization problems is given in a closed form \cite{Boyd04a}.

% 
% lemm1
% 
\begin{it}
\begin{lemm}[Ball Test]
\label{lemm:Ball.test}
Let 
 $\Theta_{[C]} \subset \cF$
 be a ball with the center 
 $\bm m \in \cF$ 
 and the radius 
 $r \in \bR_+$,
 i.e.,
$\Theta_{[C]} \triangleq \{\bm w \in \cF ~|~ \| \bm w - \bm m \| \le r\}$.
 Then,
 the lower and the upper bounds in
 \eq{eq:basic.lower.upper.bounds}
 are written as 
 \begin{eqnarray}
 \label{eq:ball.test.bounds}
 \ell_{[C]i} \equiv \min_{ \bm w \in \Theta_{[C]} } y_i f(\bm x_i; \bm w)
 =
 \bm z_i^\top \bm m - r \| \bm z_i \|, 
 ~~~
 u_{[C]i} \equiv \max_{ \bm w \in \Theta_{[C]} } y_i f(\bm x_i; \bm w)
 =
 \bm z_i^\top \bm m + r \| \bm z_i \|,
 \end{eqnarray}
 where
 we define
 $\bm z_i \triangleq y_i \Phi(\bm x_i)$, 
 $i \in \NN_n$,
 for notational simplicity.
\end{lemm}
\end{it}

\noindent
The proof is presented in Appendix \ref{app:proofs}. 
The geometric interpretation of Lemma 
\ref{lemm:Ball.test} 
is shown in 
\figurename
\ref{fig:Ball.test}.

% --------------------------------------------------
% Fig2
% --------------------------------------------------
\begin{figure}[t]
\begin{center}
\begin{tabular}{cc}
 \includegraphics[width = 0.225\textwidth]{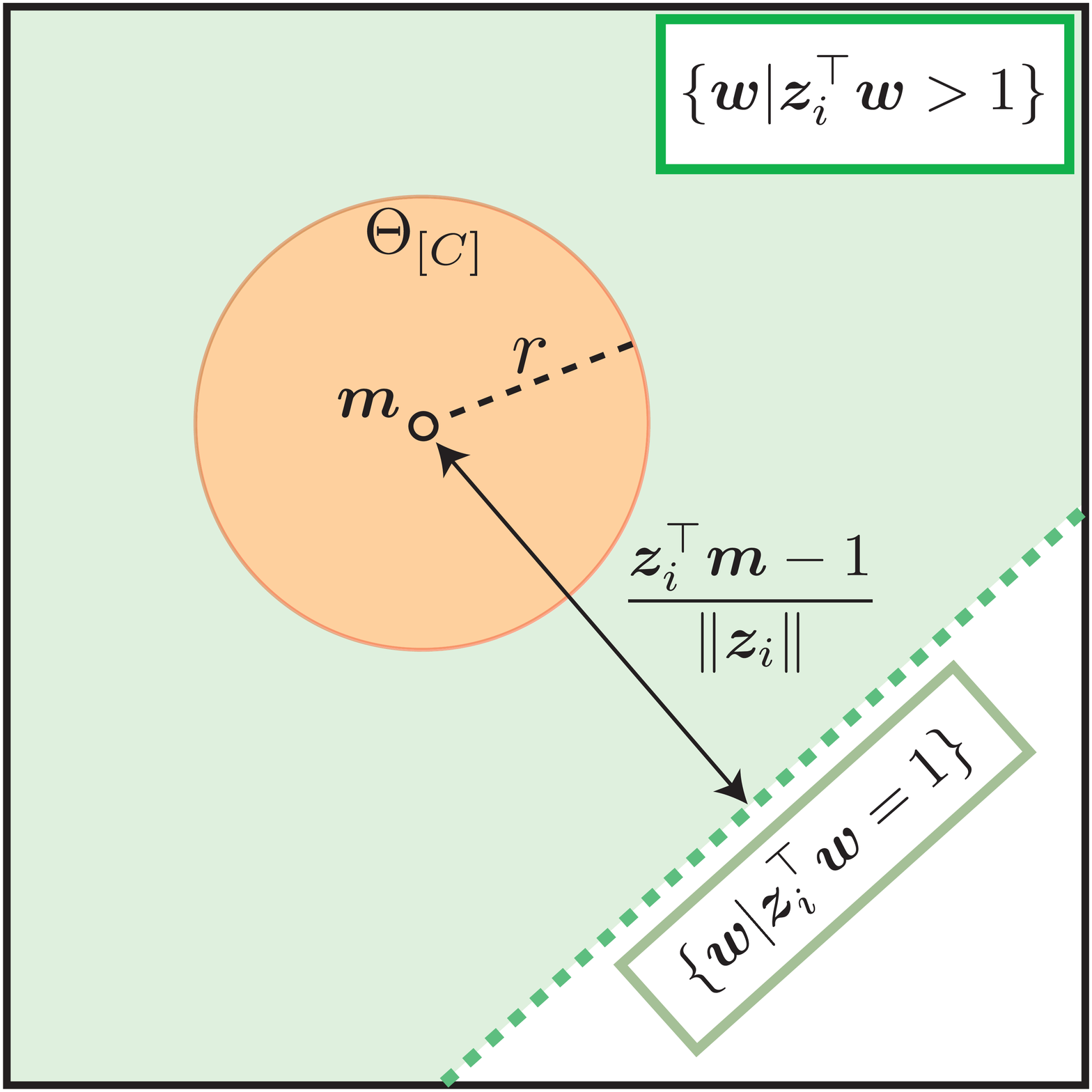} &
 \includegraphics[width = 0.225\textwidth]{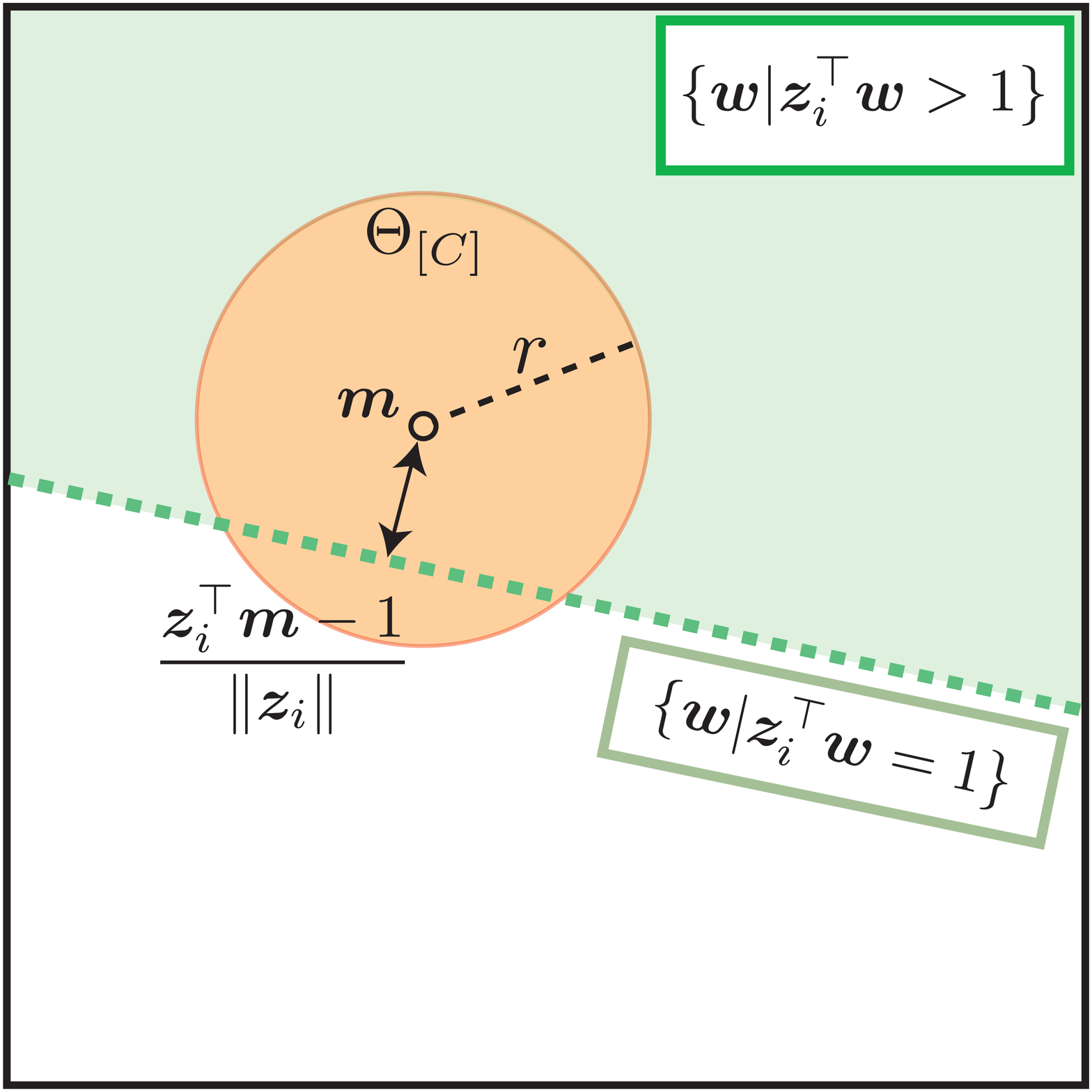} \\
 (a) Success &
 (b) Fail
\end{tabular}
 \caption{
 A geometric interpretation of ball tests.
 Two panels illustrate the solution space
 when the $i^{\rm th}$ sample
 (a) can be screened out,
 and 
 (b) cannot be screened out,
 respectively. 
 In both panels,
 the dotted green line indicates the hyperplane 
 $y_i f(\bm x_i ; \bm w) \equiv \bm z_i^\top \bm w = 1$,
 and the green region represents
 $\{\bm w | \bm z_i^\top \bm w > 1\}$.
 The orange circle
 with the center $\bm m$ and the radius $r$
 is the ball region
 $\Theta_{[C]}$
 in which the optimal solution
 $\bm w^*_{[C]}$
 exists.
 In (a),
 the fact that 
 the hyperplane 
 $\bm z_i^\top \bm w = 1$
 does not intersect with
 $\Theta_{[C]}$,
 i.e., 
 the distance
 $(\bm z_i^\top \bm m - 1)/|| \bm z_i ||$
 is larger than the radius
 $r$, 
 implies that
 $y_i f(\bm x_i ; \bm w^*_{[C]}) > 1$
 wherever the optimal solution
 $\bm w^*_{[C]}$
 locates within the region $\Theta_{[C]}$,
 and the $i^{\rm th}$ sample can be screened out as a member of $\cR$. 
 On the other hand,
 in (b),
 the hyperplane 
 $\bm z_i^\top \bm w = 1$
 intersects with
 $\Theta_{[C]}$,
 meaning that 
 we do not know whether 
 $y_i f(\bm x_i ; \bm w^*_{[C]}) > 1$
 or not
 until we actually solve the optimization problem and obtain the optimal solution
 $\bm w^*_{[C]}$.
 }
 \label{fig:Ball.test}
\end{center}
\end{figure}
% ================================================

%
% sec3.3
%
\subsection{Ball Tests for SVMs}
\label{subsec:Ball.tests.for.SVM}
The following problem is shown to be equivalent to 
\eq{eq:p.SVM.prob.}
in the sense that 
$\bm w^*_{[C]}$
is the optimal solution of the original SVM problem 
\eq{eq:p.SVM.prob.}\footnote{
Similar problem has been studied in the context of structural SVM
\cite{Joachims05,Joachims06},
and the proof of the equivalence can be easily shown by using the technique described there. 
}:
 \begin{eqnarray}
 \label{eq:str.SVM}
  %  \nonumber
  (\bm w^*_{[C]}, \xi^*_{[C]})
  \triangleq
  \arg
  \!\!\!\!\!
  \min_{\bm w \in \cF, \xi \in \bR}
  %  &&
  \!\!\!
  \cP_{[C]}(\bm w, \xi)
  \triangleq
  \frac{1}{2} \|\bm w\|^2 + C \xi 
  %\\
  ~{\rm s.t.}~
 %&&
  \xi \ge \sum_{i \in \bN_n} s_i (1-y_i f(\bm x_i)) 
  ~
  \forall \bm s \in \{0,1\}^n.
  ~
 \end{eqnarray}
We call the solution space of 
\eq{eq:str.SVM}
as
\emph{expanded solution space}.
In the expanded solution space, 
a quadratic function is minimized
over a polyhedron
composed of $2^n$ closed half spaces.

In the following lemma,
we consider a specific type of regions in the expanded solution space.
%$\{(\bm w, \xi) \in \cF \times \RR \}$.
%
%Then,
By projecting the region onto the original solution space,
%$\{\bm w \in \cF\}$,
we have a ball region in the form of Lemma \ref{lemm:Ball.test}.
%in the form of
%\eq{eq:ball.theta}.

% 
% lemm
%
\begin{it}
\begin{lemm}
\label{lemm:structural.svm.to.ball}
 Consider a region in the following form:
 \begin{eqnarray}
 \label{eq:tilde.theta}
 \Theta^\prime_{[C]}
 \triangleq
 \Big\{
 (\bm w, \xi) \in \cF \times \RR
 ~\Big|~
 a_1 \| \bm w\|^2 + \bm b_1^\top \bm w + c_1 + \xi \le 0,
 ~
 \bm b_2^\top \bm w + c_2  \le \xi
 \Big\},
 \end{eqnarray}
 where
 $a_1 \in \RR_{++}$,
 $\bm b_1, \bm b_2 \in \cF$, 
 $c_1, c_2 \in \RR$.
 If 
 $\Theta^\prime_{[C]}$
 is non-empty\footnote{
 $\Theta^\prime_{[C]}$
 is non-empty iff
 $\| \bm b_1 + \bm b_2\|^2 - 4 a_1 (c_1 + c_2) \ge 0$.
 }
 and 
 $(\bm w, \xi) \in \Theta^\prime_{[C]}$,
 $\bm w$
 is in a ball
 $\Theta_{[C]}$
 %in the form of
 %\eq{eq:ball.theta} 
 with the center
 $\bm m \in \cF$
 and the radius
 $r \in \cR_{+}$ 
 defined as 
\begin{eqnarray*}
\bm m
\triangleq
- \frac{1}{2 a_1} (\bm b_1 + \bm b_2), 
~
r
\triangleq
\sqrt{
%\frac{1}{4 a_1^2}
%\|\bm b_1 + \bm b_2\|^2
\| \bm m \|^2
-
\frac{1}{a_1}
(c_1 + c_2).
}
\end{eqnarray*}
\end{lemm}
\end{it}

\noindent
The proof is presented in Appendix \ref{app:proofs}.
The lemma suggests that a Ball Test can be constructed 
by introducing two types of necessary conditions
in the form of quadratic and linear constraints in
\eq{eq:tilde.theta}.
In the following three lemmas, 
we introduce three types of necessary conditions for the optimal solution
$(\bm w^*_{[C]}, \xi^*_{[C]})$
of the problem
\eq{eq:str.SVM}.

\begin{it}
\begin{lemm}[Necessary Condition 1 (NC1)]
\label{lemm:nc1}
 Let 
 $(\tilde{\bm w}, \tilde{\xi})$
 be a feasible solution 
 of 
 \eq{eq:str.SVM}.
 Then, 
\begin{eqnarray}
 \label{eq:nc1}
  \frac{1}{C}
 \| \bm w^*_{[C]} \|^2
 -
 \frac{1}{C}
 \tilde{\bm w}^\top
 \bm w^*_{[C]}
 - 
 \tilde{\xi}
 +
 \xi^*_{[C]}
 \le
 0.
 %\red{
 %\text{ where }
 %\tilde{\xi} \triangleq\sum_{i \in \NN_n} \max \left(0, 1 - y_i f(\bm x_i ; \tilde{\bm w}) \right).
 %}
 \end{eqnarray}
\end{lemm}
\end{it}

\begin{it}
\begin{lemm}[Necessary Condition 2 (NC2)]
\label{lemm:nc2}
 Let 
 %$\bm w^*_{[\check{C}]}$
 $(\bm w^*_{[\check{C}]}, \xi^*_{[\check{C}]})$
 be the optimal solution for any other regularization parameter
 $\check{C} \in \RR_{++}$.
 Then, 
 \begin{eqnarray}
  \label{eq:nc2}
   - \frac{1}{\check{C}}
   \bm w_{[\check{C}]}^{* \top} \bm w^*_{[C]}
   +
   \frac{1}{\check{C}} \| \bm w^*_{[\check{C}]} \|^2 + \xi^*_{[\check{C}]} 
  \le \xi^*_{[C]}.
 %\red{
 %\text{ where }
 %\xi^*_{[\check{C}]} \triangleq\sum_{i \in \NN_n} \max \left(0, 1 - y_i f(\bm x_i ; \bm w^*_{[\check{C}]}) \right).
 %}
 \end{eqnarray}
\end{lemm}
\end{it}

\begin{it}
\begin{lemm}[Necessary Condition 3 (NC3)]
\label{lemm:nc3}
 Let 
 $\hat{\bm s} \in \{0, 1\}^n$
 be an $n$-dimensional binary vector.
 Then, 
 \begin{eqnarray}
  \label{eq:nc3}
   - \bm z_{\hat{\bm s}}^\top \bm w^*_{[C]}
  +
   \hat{\bm s}^\top \ones
  \le
  \xi^*_{[C]},
  \text{ where }
 \bm z_{\hat{\bm s}} \triangleq \sum_{i \in \NN_n} \hat{s}_i \bm z_i. 
 \end{eqnarray}
\end{lemm}
\end{it}

\noindent
The proofs of these three lemmas are presented in Appendix \ref{app:proofs}.
Note that
NC1 is quadratic,
while
NC2 and NC3 are linear constraints 
in the form of
\eq{eq:tilde.theta}.
As described in the following theorems, 
\emph{Ball Test 1 (BT1)}
is constructed by using 
NC1 
and 
NC2,
while 
\emph{Ball Test 2 (BT2)}
is constructed
by using 
NC1
and 
NC3.

%
% theo6:BT1
%
\begin{it}
\begin{theo}[Ball Test 1 (BT1)]
 Let
 %$\tilde{\bm w}$
 $(\tilde{\bm w}, \tilde{\xi})$
 be any feasible solution
 and
 %$\bm w^*_{[\check{C}]}$
 $(\bm w^*_{[\check{C}]}, \xi^*_{[\check{C}]})$
 be the optimal solution
 of 
 \eq{eq:str.SVM}
 for any other regularization parameter
 $\check{C}$.
 Then, 
 the optimal SVM solution
 $\bm w^*_{[C]}$
 is included in the ball
 $\Theta^{\rm (BT1)}_{[C]}
 \triangleq
 \{ \bm w ~\big|~ \| \bm w - \bm m_1 \| \le r_1 \}$,
 where
 \begin{eqnarray}
  \label{eq:BT1.m.r}
  \bm m_1
  \triangleq
  \frac{1}{2}(\tilde{\bm w} + \frac{C}{\check{C}} \bm w^*_{[\check{C}]}),
  ~
  r_1
  \triangleq
  \sqrt{
  %\frac{1}{4} \| \tilde{\bm w} + \frac{C}{\check{C}} \bm w^*_{[\check{C}]} \|^2
  \| \bm m_1\|^2
  - \frac{C}{\check{C}} \|\bm w^*_{[\check{C}]}\|^2
  + C(\tilde{\xi} - \xi^*_{[\check{C}]})
  }.
 \end{eqnarray}
 By applying the ball
 $\Theta^{\rm (BT1)}_{[C]}$
 to Lemma \ref{lemm:Ball.test},
 we can compute the lower bound
 $\ell^{\rm (BT1)}_{[C]}$
 and the upper bound
 $u^{\rm (BT1)}_{[C]}$.
\end{theo}
\end{it}

% 
% theo7:BT2
%
\begin{it}
\begin{theo}[Ball Test 2 (BT2)]
 Let
%  $\tilde{\bm w}$
$(\tilde{\bm w}, \tilde{\xi})$
 be any feasible solution of 
 \eq{eq:str.SVM}
 and
 $\hat{\bm s}$
 be any $n$-dimensional binary vector
 in $\{0, 1\}^n$.
 Then, 
 the optimal SVM solution
 $\bm w^*_{[C]}$
 is included in the ball
  $\Theta^{\rm (BT2)}_{[C]} \triangleq \{ \bm w ~\big|~ \| \bm w - \bm m_2 \| \le r_2 \}$,
 where
 \begin{eqnarray*}
  \bm m_2
  \triangleq
  \frac{1}{2} (\tilde{\bm w} + C \bm z_{\hat{\bm s}}),
  ~
  r_2
  \triangleq
  \sqrt{
  %\frac{1}{2} \| \tilde{\bm w} + C \bm z_{\hat{\bm s}}\|^2
  \| \bm m_2\|^2
  + C( \tilde{\xi} - \hat{\bm s}^\top \ones)
  }.
  %\sqrt{
  %\|\bm m_1\|^2
  %+
  %2 \bigl(
  %\rho - {\frac{C}{\check{C}}}
  %(\| \bm w_{[\check{C}]}^*\|^2 + \check{C} \xi_{[\check{C}]}^*)
  %\bigr)
  %}.
 \end{eqnarray*}
 By applying the ball
 $\Theta^{\rm (BT2)}_{[C]}$
 to Lemma \ref{lemm:Ball.test},
 we can compute the lower bound
 $\ell^{\rm (BT2)}_{[C]}$
 and the upper bound
 $u^{\rm (BT2)}_{[C]}$.
\end{theo}
\end{it}

%We denote the lower and the upper bounds obtained by BT1 and BT2 as 
%$\ell_{[C]i}^{\rm (BT1)}$,
%$u_{[C]i}^{\rm (BT1)}$, 
%$\ell_{[C]i}^{\rm (BT2)}$
%and 
%$u_{[C]i}^{\rm (BT2)}$, 
%respectively.

%
% sec3.4: intersection test (IT)
%
\subsection{Intersection Test}
\label{subsec:Intersection.Test}
We introduce a more powerful screening test  
called 
\emph{Intersection Test (IT)}
based on 
\begin{eqnarray*}
\Theta_{[C]}^{\rm (IT)} \triangleq \Theta_{[C]}^{\rm (BT1)} \cap \Theta_{[C]}^{\rm (BT2)}.
\end{eqnarray*}

%
% Theo8 (IT)
%
\begin{it}
\begin{theo}[Intersection Test]
\label{theo:IT}
The lower and the upper bounds of 
$y_i f(\bm x_i; \bm w)$ 
in 
$\Theta_{[C]}^{\rm (IT)}$
are 
 \begin{eqnarray}
  \label{eq:it.test.low}
  \ell^{\rm (IT)}_{[C]i}
  \triangleq
  \min_{\bm w \in \Theta_{[C]}^{\rm (IT)}} y_i f(\bm x_i ; \bm w) = 
  \mycase{
  \ell^{\rm (BT1)}_{[C]i}
  &
  \text{ if }
  \frac{-\bm z_i^\top \bm \phi}{\| \bm z_i \|~\| \bm \phi \| } < \frac{\zeta - \| \bm \phi\|}{r_1},
  \\
  \ell^{\rm (BT2)}_{[C]i} 
  &
  \text{ if }
  \frac{\zeta}{r_2} < \frac{-\bm z_i^\top \bm \phi}{\| \bm z_i \|~\| \bm \phi \|}, 
  \\
  \bm z_i^\top \bm \psi - \kappa \sqrt{\| \bm z_i \|^2 - \frac{(\bm z_i^\top \bm \phi)^2}{\| \bm \phi \|^2}}
   & 
   \text{ if } 
   \frac{\zeta - \| \bm \phi \|}{r_1} 
   \le \frac{-\bm z_i^\top \bm \phi}{\| \bm z_i \|~\| \bm \phi \| } 
   \le \frac{\zeta}{r_2}
   }   
\end{eqnarray}
and
\begin{eqnarray}
\label{eq:it.upp}
u^{\rm (IT)}_{[C]i}
\triangleq
\max_{\bm w \in \Theta_{[C]}^{\rm (IT)}} y_i f(\bm x_i ; \bm w) = 
\mycase{
u^{\rm (BT1)}_{[C]i}
&
\text{ if }
\frac{\bm z_i^\top \bm \phi}{\| \bm z_i \|~\| \bm \phi \| }
<
\frac{\zeta - \| \bm \phi\|}{r_1},
\\
u^{\rm (BT2)}_{[C]i} 
&
\text{ if }
\frac{\zeta}{r_2}
<
\frac{\bm z_i^\top \bm \phi}{\| \bm z_i \|~\| \bm \phi \| },
\\
\bm z_i^\top \bm \psi + \kappa \sqrt{\| \bm z_i \|^2 - \frac{(\bm z_i^\top \bm \phi)^2}{\| \bm \phi \|^2}}
& 
\text{ if } 
\frac{\zeta - \| \bm \phi \|}{r_1} 
\le \frac{\bm z_i^\top \bm \phi}{\| \bm z_i \|~\| \bm \phi \| } 
\le \frac{\zeta}{r_2},
} 
\end{eqnarray}
where
\begin{eqnarray*}
\bm \phi
\triangleq
\bm m_1 - \bm m_2,
~
\zeta
\triangleq
\frac{1}{2 \| \bm \phi \| } ( \| \bm \phi \|^2 + r_2^2 - r_1^2 ),
~
\bm \psi
\triangleq \bm m_2 + \zeta \bm \phi / \|\bm \phi\| ,
~
\kappa
\triangleq \sqrt{ r_2^2 - \zeta^2}.
\end{eqnarray*}
\end{theo}
\end{it}

\noindent
The proof is presented in Appendix \ref{app:proofs}. 
Note that IT is guaranteed to be more powerful than BT1 and BT2
because
$\Theta^{\rm (IT)}_{[C]}$
is the intersection of 
$\Theta^{\rm (BT1)}_{[C]}$
and
$\Theta^{\rm (BT2)}_{[C]}$.

%
% sec4
%

% --------------------------------------------------
% Fig3
% --------------------------------------------------
\begin{figure}[t]
\begin{center}
\includegraphics[width = 0.35\textwidth]{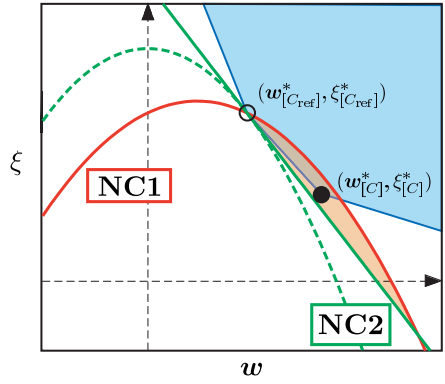}
\caption{
 A schematic illustration of the two necessary conditions NC1 and NC2 in
 the expanded solution space 
 when we use the reference solution
 $(\bm w^*_{[C_{\rm ref}]}, \xi^*_{[C_{\rm ref}]})$.
 The blue polytope in the upper-right corner
 indicates the feasible region of 
 \eq{eq:str.SVM}.
 %
% The red and green dotted quadratic curves indicate the objective functions 
% $\cP_{[C]}(\bm w, \xi)$
% and 
% $\cP_{[C_{\rm ref}]}(\bm w, \xi)$,
% respectively.
 %
 The open circle
 $\circ$
 and the filled circle 
 $\bullet$
 indicate the optimal solutions
 $(\bm w^*_{[C_{\rm ref}]}, \xi^*_{[C_{\rm ref}]})$
 and
 $(\bm w^*_{[C]}, \xi^*_{[C]})$,
 respectively. 
 The red quadratic curve and green line indicate the boundaries of NC1 and NC2, respectively. 
 Note that the green line is the tangent at the point 
 $(\bm w^*_{[C_{\rm ref}]}, \xi^*_{[C_{\rm ref}]})$ 
 of the objective function 
 $\cP_{[C_{\rm ref}]}(\bm w, \xi)$
 which is shown by green dotted quadratic curve. 
 The area surrounded by the red quadratic curve and the green line is the
 region
 $\Theta^\prime_{C}$
 in which the optimal solution 
 $(\bm w^*_{[C]}, \xi^*_{[C]})$,
 exists.
 %
 %NC1 suggests that
 %the optimal solution with $C$
 %should have smaller objective value than
 %the optimal solution with $C_0$,
 %i.e.,
 %the optimal solution
 %$(\bm w^*_{[C]}, \xi^*_{[C]})$
 %must be located in the orange area. 
 %%
 %(b)
 %In the right panel,
 %the gradient vector
 %$\bm \nabla \cP_{[C_0]}(\bm w^*_{[C_0]}, \xi^*_{[C_0]})$
 %which points the direction of increasing the objective 
 %$\cP_{[C]}(\bm w, \xi)$
 %is shown.
 %%
 %NC2 suggests that
 %the optimal solution 
 %$(\bm w^*_{[C]}, \xi^*_{[C]})$
 %must be located in the green area,
 %i.e., 
 %the angle between 
 %the directions from the optimal point 
 %$(\bm w_{[C_0]}^*,\xi_{[C_0]}^*)$ 
 %to a feasible point 
 %$(\bm w_{[C]}^*,\xi_{[C]}^*)$ 
 %and 
 %the gradient vector 
 %$\bm \nabla \cP_{[C_0]}(\bm w^*_{[C_0]}, \xi^*_{[C_0]})$
 %must be
 %$\le \pi/2$.
}
\label{fig:theta}
\end{center}
\end{figure}

 \section{Safe Sample Screening in Practice}
\label{sec:in.practice}
In order to use the safe sample screening methods in practice,
we need two additional side information: 
a feasible solution
$(\tilde{\bm w}, \tilde{\xi})$
and 
the optimal solution 
$(\bm w^*_{[\check{C}]}, \xi^*_{[\check{C}]})$
for a different regularization parameter
$\check{C}$.
Hereafter, 
we focus on a particular situation that the optimal solution
$\bm w^*_{[C_{\rm ref}]}$
for a smaller
$C_{\rm ref} < C$
is available,
and call such a solution as 
\emph{a reference solution}. 
We later see that such a reference solution can be easily available in practical model building process.
Let
$
\xi^*_{ [C_{\rm ref}] }
\triangleq
\sum_{i \in \NN_n}
\max
\{
0,
1 - y_i f( \bm x_i ; \bm w^*_{ [C_{\rm ref}] })
\}
$.
By replacing both of 
$(\tilde{\bm w}, \tilde{\xi})$
and 
$(\bm w^*_{[\check{C}]}, \xi^*_{[\check{C}]})$ 
with 
$(\bm w^*_{[C_{\rm ref}]}, \xi^*_{[C_{\rm ref}]})$, 
the centers and the radiuses of
$\Theta^{\rm (BT1)}_{[C]}$
and 
$\Theta^{\rm (BT1)}_{[C]}$
are rewritten as
\begin{eqnarray*}
 \bm m_1 = \frac{C + C_{\rm ref}}{2 C_{\rm ref}} \bm w^*_{[C_{\rm ref}]},
 r_1 = \frac{C - C_{\rm ref}}{2 C_{\rm ref}} \| \bm w^*_{[C_{\rm ref}]}\|,
\end{eqnarray*}
\begin{eqnarray*}
 \bm m_2 = \frac{1}{2}(\bm w^*_{[C_{\rm ref}]} + C \bm z_{\hat{\bm s}}), 
 r_2 = \sqrt{\| \bm m_2 \|^2 + C(\xi^*_{[C_{\rm ref}]} - \hat{\bm s}^\top \bm 1)}.
\end{eqnarray*}
A geometric interpretation of the two necessary conditions 
NC1 
and 
NC2 
in this special case is illustrated in \figurename 
\ref{fig:theta}. 
In the rest of this section,
we discuss how to obtain reference solutions and other practical issues.

%In \S \ref{subsec:how.to.compute.reference.solution},
%we show that
%a reference solution can be trivially obtained
%for a sufficiently small regularization parameter.
%%
%In \S \ref{subsec:regularization.path},
%we discuss
%\emph{regularization path} computation scenario,
%where a sequence of more effective reference solutions can be naturally obtained. 
%%
%In \S \ref{subsec:specific.s},
%we present how to select 
%$\hat{\bm s} \in \{0, 1\}^n$.
%%
%In \S \ref{subsec:kernelization},
%we show that our methods can be \emph{kernelized}.
%%
%In \S \ref{subsec:complexity},
%we summarize computational costs.
%%
%In \S \ref{subsec:relation.with.feature.screening},
%we discuss relationships with existing methods. 

%
% sec4.1
% 
\subsection{How to obtain a reference solution}
\label{subsec:how.to.compute.reference.solution}
The following lemma implies that, 
for a sufficiently small regularization parameter 
$C$, 
we can make use of a trivially obtainable reference solution.

% -------------------------
% lemm7
% -------------------------
\begin{it}
\begin{lemm}
\label{lemm:C.min}
Let 
% \begin{eqnarray}
%  \label{eq:C.min}
%   C_{\rm min} \triangleq 
%   \frac{1}{ \displaystyle\max_{i \in \bN_n} (\bm Q \ones)_i}. 
% \end{eqnarray}
$C_{\rm min} \triangleq 1/\max_{i \in \bN_n} (\bm Q \ones)_i$.
 Then, 
 for 
 $C \in (0,C_{\rm min}]$, 
 the optimal solution of the dual SVM formulation 
 \eq{eq:d.SVM.prob.}
 is written as
 $\bm \alpha^*_{[C]} = C \ones$.
 % \begin{eqnarray*}
 %  \bm \alpha^*_{[C]}
 %  =
 %  C \ones.
   % \end{eqnarray*}%
\end{lemm} 
\end{it}
The proof is presented in Appendix \ref{app:proofs}. 
Without loss of generality,
we only consider the case with 
$C > C_{\rm min}$,
where we can use the solution 
$\bm w^*_{[C_{\rm min}]}$
as the reference solution.

%
% sec4.2
% 
\subsection{Regularization path computation}
\label{subsec:regularization.path}
In model selection process, a sequence of SVM classifiers with various
different regularization parameters $C$ are trained.
Such a sequence of the solutions is sometimes referred to as
\emph{regularization path}
\cite{Hastie04a,Giesen12b}.
%\footnote{
%In a narrow sense, \emph{regularization path} refers to the continuous
%path of the optimal solutions parametrized by the regularization
%parameter $C$. In this paper, we use the term in a broad sense to refer
%to a sequence of SVM solutions with various different regularization
%parameters.
%}. 
%
%In regularization path computation scenario,
%the proposed method can enjoy additional computational advantages. 
%
Let us write the sequence as
$C_1 < \ldots < C_T$.
We note that 
SVM is easier to train
(the convergence tends to be faster)
for smaller regularization parameter
$C$. 
Therefore,
it is reasonable to compute the regularization path from smaller $C$ to larger $C$
with the help of 
\emph{warm-start} approach
\cite{decoste00a},
where 
the previous optimal solution at 
$C_{t - 1}$
is used as the initial starting point of the next optimization problem 
for 
$C_t$.
In such a situation, 
we can make use of the previous solution
at
$C_{t - 1}$
as the reference solution.
Note that this is more advantageous
than using 
$C_{\rm min}$
as the reference solution
because the rules can be more powerful 
when the reference solution
is closer to
$\bm w^*_{[C]}$.
Moreover,
the rule evaluation cost can be reduced in regularization path computation scenario
(see \ref{subsec:complexity}). 

%\red{
%In our experiments (\S \ref{sec:experiments}),
%we use a recently proposed
%$\eps$-approximation path alogorithm
%\cite{},
%where a sequence of the regularization parameters
%$C_1, \ldots, C_T$
%is determined in such a way that
%the relative difference of the objective values between
%$C_{t - 1}$
%and
%$C_{t}$
%is upper bounded by
%$\eps$
%(a small constant such as $\eps = 10^{-3}$).
%%
%The proposed safe sample screening rules tend to work better 
%when the objective value of the reference solution is not much different from the current one. 
%}

% 
% sec4.3
% 
\subsection{How to select $\hat{\bm s}$ for the necessary condition 3}
\label{subsec:specific.s}
We discuss how to select
$\hat{\bm s} \in \{0, 1\}^n$
for NC3. 
Since a smaller region leads to a more powerful rule, 
it is reasonable to select 
$\hat{\bm s} \in \{0, 1\}^n$
so that 
the volume of the intersection region
$\Theta_{[C]}^{\rm (IT)} \equiv \Theta_{[C]}^{\rm (BT1)} \cup \Theta_{[C]}^{\rm (BT2)}$
is as small as possible.
%
%Since the problem of selecting the 
%$\hat{\bm s} \in \{0, 1\}^n$
%that minimizes the volume
%$\Theta_{[C]}^{\rm (IT)}$
%is a combinatorial hard problem, 
%we introduce a heuristic for selecting a reasonably good
%$\hat{\bm s}$
%in a small computation cost. 
%
%Specifically,
We select
$\hat{\bm s}$
such that 
the distance between the two balls
$\Theta^{\rm (BT1)}_{[C]}$
and
$\Theta^{\rm (BT2)}_{[C]}$
is maximized,
while 
the radius of 
$\Theta^{\rm (BT2)}_{[C]}$
is minimized,
%\footnote{
%Note that the radius of
%$\Theta^{\rm (BT1)}_{[C]}$
%is fixed.
%},
i.e.,
\begin{eqnarray}
\label{eq:hat.s.selection}
\hat{\bm s} = \arg\max_{\bm s \in \{0,1\}^n}
\left( \| \bm m_1 - \bm m_2 \|^2 - r_2^2 \right)
=
\arg\max_{\bm s \in \{0,1\}^n} \sum_{i \in \bN_n} s_i (1- \frac{C + C_{\rm ref}}{2C_{\rm ref}} y_i f(\bm x_i;\bm w_{[C_{\rm ref}]}^*)). 
\end{eqnarray}
Note that the solution of
\eq{eq:hat.s.selection}
can be straightforwardly obtained as 
\begin{eqnarray*}
 \hat{s}_i = I\{1 - \frac{C + C_{\rm ref}}{2C_{\rm ref}} y_i f(\bm x_i;\bm w_{[C_{\rm ref}]}^*) > 0\},
 ~
 i \in \NN_n,
\end{eqnarray*}
where
$I(\cdot)$
is the indicator function.

% 
% sec4.4
% 
 \subsection{Kernelization}
\label{subsec:kernelization}
The proposed safe sample screening rules can be
\emph{kernelized},
i.e.,
all the computations can be carried out
without explicitly working on the high-dimensional feature space
$\cF$.
Remembering that
$Q_{ij} = \bm z_i^\top \bm z_j \equiv y_i \Phi(\bm x_i)^\top \Phi(\bm x_j) y_j$,
we can rewrite the rules by using the following relations:
\begin{eqnarray*}
 \| \bm z_i \| = \sqrt{Q_{ii}},~
 \| \bm w_{[C_{\rm ref}]}^* \| = \sqrt{\bm \alpha_{[C_{\rm ref}]}^{*\top} \bm Q \bm \alpha_{[C_{\rm ref}]}^*},~
 \bm z_i^\top \bm m_1 = \frac{C + C_{\rm ref}}{2 C_{\rm ref}} (\bm Q \bm \alpha^*_{[C_{\rm ref}]})_i,
\end{eqnarray*}
\begin{eqnarray*}
 \bm z_i^\top \bm m_2 = \frac{1}{2} (\bm Q \bm \alpha^*_{[C_{\rm ref}]})_i + \frac{C}{2} (\bm Q \hat{\bm s})_i,~
 \|\bm m_1\| = \frac{C + C_{\rm ref}}{2 C_{\rm ref}} \sqrt{\bm \alpha^{*\top}_{[C_{\rm ref}]} \bm Q \bm \alpha^{*}_{[C_{\rm ref}]}},
\end{eqnarray*}
\begin{eqnarray*}
\|\bm m_2\| = \frac{1}{2} \sqrt{(\bm \alpha^{*}_{[C_{\rm ref}]} + C \hat{\bm s})^\top \bm Q (\bm \alpha^{*}_{[C_{\rm ref}]} + C \hat{\bm s})},~
\bm m_1^\top \bm m_2 = \frac{C + C_{\rm ref}}{4 C_{\rm ref}} (\bm \alpha^{*\top}_{[C_{\rm ref}]} \bm Q \bm \alpha^{*}_{[C_{\rm ref}]} + C \bm \alpha^{*\top}_{[C_{\rm ref}]} \bm Q \hat{\bm s}).
\end{eqnarray*}
Exploiting the sparsities of 
$\bm \alpha_{[C_{\rm ref}]}^*$ 
and 
$\hat{\bm s}$,
some parts of the rule evaluations can be done efficiently
(see \S \ref{subsec:complexity} for details).

% 
% sec4.5
% 
\subsection{Computational Complexity}
\label{subsec:complexity}
The computational complexities for evaluating the safe sample screening rules 
are summarized in 
Table \ref{tbl:complexity}. 
Note that the rule evaluation cost can be reduced in regularization path computation scenario.
The bottleneck of the rule evaluation is in the computation of 
$\bm \alpha_{[C_{\rm ref}]}^{*\top} \bm Q \bm \alpha_{[C_{\rm ref}]}^*$. 
Since many SVM solvers 
(including {\liblinear} and {\libsvm})
use the value 
$\bm Q \bm \alpha$ 
in their internal computation and store it in a cache,
we can make use of the cache value for circumventing the bottleneck. 
Furthermore,
BT2 (and henceforth IT)
can be efficiently computed in regularization path computation scenario by
caching
$\bm Q \hat{\bm s}$.

%
% tab1
%
\begin{table}[h]
 \label{tbl:complexity}
 \caption{The computational complexities of the rule evaluations}
 \begin{center}
  \begin{tabular}{l|l|l|l}
   & linear & kernel	& kernel (cache) 
   \\ \hline
   BT1	& $\mathcal{O}(n d_s)$	& $\mathcal{O}(n^2)$	&	$\mathcal{O}(n)$ 
\\ 
BT2	& $\mathcal{O}(n d_s)$	& $\mathcal{O}(n^2)$ & $\mathcal{O}(n \|\Delta \hat{\bm s}\|_0)$ 
\\
IT	& $\mathcal{O}(n d_s)$	& $\mathcal{O}(n^2)$	& $\mathcal{O}(n \|\Delta \hat{\bm s}\|_0)$ 
\\ 
  \end{tabular} 
  \end{center}
%  The computational complexities
%  in evaluating the safe screening rules
%  for all the $n$ samples are presented.  
 %
 For each of
 Ball Test 1 (BT1),
 Ball Test 2 (BT2),
 and
 Intersection Test (IT), 
 the complexities
 for evaluating the safe sample screening rules for all
 $i \in \NN_n$
 of
 linear SVM 
 and
 nonlinear kernel SVM 
 (with and without using the cache values
 as discussed in \S \ref{subsec:regularization.path})
 are shown. 
Here, 
$d_s$ 
indicates the average number of non-zero features for each sample
 and 
 $\|\Delta \hat{\bm s}\|_0$
 indicates the number of different elements in
 $\hat{\bm s}$
 between
 two consecutive 
 $C_{t - 1}$
 and 
 $C_t$
 in regularization path computation scenario.
\end{table}

% -------------------------
% sec4.6
% -------------------------
\subsection{Relation with existing approaches}
\label{subsec:relation.with.feature.screening}
This work is highly inspired by the 
\emph{safe feature screening} 
introduced by El Ghaoui et al. 
\cite{ElGhaoui12b}.
%
%The tests based on a ball region $\Theta$ have also been used in the context of safe feature screening
%\cite{Xiang12a}.
%
After the seminal work by El Ghaoui et al. \cite{ElGhaoui12b}, 
many efforts have been devoted for improving screening performances
\cite{Xiang12a,Xiang12b,Dai12a,Wang12a,Wang13c,Wang13d,Wu13a,Wang13a,Wang13b}.
All the above listed studies are designed for screening the features
in $L_1$ penalized linear model\footnote{
El Ghaoui et al. \cite{ElGhaoui12b}
also studied safe feature screening for $L_1$-penalized SVM.
Note that their work is designed for screening features based on the property of $L_1$ penalty,
and it cannot be used for sample screening.}~\footnote{
Jaggie et al. \cite{Jaggie13a}
discussed the connection between LASSO and ($L_2$-hinge) SVM,
where they had an comment that
the techniques used in safe feature screening for LASSO might be also useful in the context of SVM.
}.

As the best of our knowledge, 
the approach presented in our conference paper \cite{Ogawa13a}
is the first safe sample screening method 
that can
\emph{safely}
eliminate a subset of the samples
before actually solving the training optimization problem. 
Note that this extension is non-trivial
because the feature sparseness in a linear model 
stems from the $L_1$ penalty,
while the sample sparseness in an SVM
is originated from the large-margin principle.

After our conference paper \cite{Ogawa13a} was published, 
Wang et al.
\cite{Wang13e}
recently proposed a method called
\emph{DVI test},
and showed that it is more powerful than our previous method in 
\cite{Ogawa13a}.
In this paper, we further go beyond the DVI test. 
We can show that DVI test 
is equivalent to Ball Test 1 (BT1)
in a special case
(the equivalence is shown in Appendix \ref{app:equivalence.DVI}).
Since 
the region
$\Theta^{\rm (IT)}_{[C]}$
is included in the region
$\Theta^{\rm (BT1)}_{[C]}$, 
Intersection Test (IT) is theoretically guaranteed to be more powerful than DVI test.
We will also empirically demonstrate that IT consistently outperforms DVI test
in terms of screening performances in \S \ref{sec:experiments}.

One of our non-trivial contributions is in 
\S \ref{subsec:Ball.tests.for.SVM},
where a ball-form region is constructed 
by first considering a region in the expanded solution space
and then projecting it onto the original solution space.
The idea of merging two balls for constructing the intersection region
in \S \ref{subsec:Intersection.Test}
is also our original contribution. 
We conjecture that 
the basic idea of Intersection Test can be also useful for safe \emph{feature} screening.

%Original safe feature screening rule in \cite{ElGhaoui12b} 
%assumes the availability of two types of reference solutions
%(one for smaller and another for larger regularization parameters)
%which can be easily obtained 
%in Lasso feature screeening case. 
%
%The main difficulty in SVM safe sample screening case is that 
%it is difficult to get the latter reference solution.
%In our preliminary conference paper,
%we had to consider a bit restricited setup 
%because of this reason
%(see Appendix \ref{app:v.s.ICML.ver.}).

\section{Experiments}
\label{sec:experiments}
We demonstrate the advantage of the proposed safe sample screening methods
through numerical experiments.
We first describe the problem setup of
\figurename \ref{fig:toy.example} 
in
\S \ref{subsec:toy.experiment}.
In \S \ref{subsec:screening.rate},
we report the screening rates,
i.e.,
how many percent of the non-SVs can be screened out by safe sample screening.
In \S \ref{subsec:computation.time},
we show that the computational cost of the state-of-the-art SVM solvers
({\libsvm}
\cite{Chang11a}
and 
{\liblinear}
\cite{Fan08b}\footnote{
Since
the original
{\libsvm}
cannot be used for the model
without \emph{bias} term,
%\eq{eq:d.model}, 
we slightly modified the code, 
while 
we used 
{\liblinear}
as it is because 
it is originally designed for models without bias term.}
)
can be substantially reduced with the use of safe sample screening.
%
%In this section, 
%we show the results on Ball Test 1 (BT1) and Intersection Test (IT).
%
Note that DVI test proposed in
\cite{Wang13e}
is identical with BT1
in all the experimental setups considered here
(see Appendix \ref{app:equivalence.DVI}).
Table \ref{tbl:data.set}
summarizes the benchmark data sets used in our experiments.

%
% tab2
%
\begin{table}[h]
\caption{Benchmark data sets used in the experiments}
\label{tbl:data.set}
\begin{center}
 \begin{tabular}{l|r|r}
Data Set & \#samples ($n$) & \#features ($d$) \\ 
\hline \hline
D01: B.C.D		& 569 & 30 \\ 
D02: dna		& 2,000 & 180 \\ 
D03: DIGIT1		& 1,500 & 241 \\ 
D04: satimage	& 4,435 & 36 \\ 
\hline
D05: gisette		& 6,000 & 5,000 \\ 
D06: mushrooms	& 8,124 & 112 \\ 
D07: news20	& 19,996 & 1,355,191 \\
D08: shuttle		& 43,500 & 9 \\
\hline
D09: acoustic	& 78,832 & 50 \\ 
%
%D10: covtype		& 581,012 & 54 \\ 
% 
%D11: yahoo		& 1,036,492 & 700 \\ 
%
D10: url			& 2,396,130 & 3,231,961 \\ 
D11: kdd-a		& 8,407,752 & 20,216,830 \\ 
D12: kdd-b		& 19,264,097 & 29,890,095 \\ 
 \end{tabular} 
\end{center}

 \vspace*{2.5mm}
 
  We refer
  D01 $\sim$ D04
  as \emph{small},
  D05 and D08
  as \emph{medium},
  and D09 $\sim$ D12
  as \emph{large} data sets. 
  We only used linear kernel for large data sets
  because the kernel matrix computation for 
  $n > 50,000$
  is computationally prohibitive.
\end{table}

% -------------------------
% sec5.1
% -------------------------
 \subsection{Artificial toy example in \figurename \ref{fig:toy.example}}
\label{subsec:toy.experiment}
The data set
$\{(\bm x_i, y_i)\}_{i \in \NN_{1000}}$
in 
\figurename \ref{fig:toy.example}
was generated as
\begin{eqnarray*}
 &&
  \bm x_i
  \sim
  N([-0.5, -0.5]^\top, 1.5^2\bm{I})
  \text{ and }
  y_i = -1
  ~\text{for odd}~
  i,
  \\
  &&
  \bm x_i
  \sim
  N([+0.5, +0.5]^\top, 1.5^2\bm{I})
  \text{ and }
  y_i = +1
  ~\text{for even}~
  i,
\end{eqnarray*}
where
$\bm I$
is the identity matrix.
We considered the problem of learning a linear classifier
at $C = 10$.
Intersection Test was conducted by using the reference solution at 
$C_{\rm ref} = 5$.
For the purpose of illustration,
\figurename \ref{fig:toy.example}
only highlights the area
in which the samples are screened out as the members of $\cR$
(red and blue shaded regions).

 % 
 % sec5.2
 % 
 \subsection{Screening rate}
 \label{subsec:screening.rate}
 We report the screening rates of BT1, BT2 and IT. 
 The screening rate is defined as the number of the screened samples over the total number of the non-SVs
 (both in $\cR$ and $\cL$).
 The rules were constructed by using the optimal solution at 
 $C_{\rm ref} (< C)$
 as the reference solution. 
 We used linear kernel and RBF kernel
 $K(\bm x, \bm x^\prime) = \exp(-\gamma \|\bm x - \bm x^\prime\|^2)$
 where
 $\gamma \in \{0.1/d, 1/d, 10/d\}$
 is a kernel parameter and $d$ is the input dimension.

 Due to the space limitation,
 we only show the results on four small data sets with $C = 10$ in
 \figurename \ref{fig:screening.rate}.
 In each plot,
 the horizontal axis denotes
 $C_{\rm ref}/C \in (0, 1]$.
 In most cases, 
 the screening rates increased as
 $C_{\rm ref}/C$
 increases from 0 to 1,
 i.e., 
 the rules are more powerful when the reference solution
 $\bm w^*_{[C_{\rm ref}]}$
 is closer to
 $\bm w^*_{[C]}$.
 The screening rates of IT were always higher than those of BT1 and BT2
 because
 $\Theta^{\rm (IT)}_{[C]}$
 is shown to be smaller than 
 $\Theta^{\rm (BT1)}_{[C]}$
 and 
 $\Theta^{\rm (BT2)}_{[C]}$
 by construction.
 The three tests behaved similarly in other problem setups.

% Fig4
\begin{figure}[t]
 \begin{center}
  \begin{tabular}{cccc}
   \includegraphics[width = 0.225\textwidth]{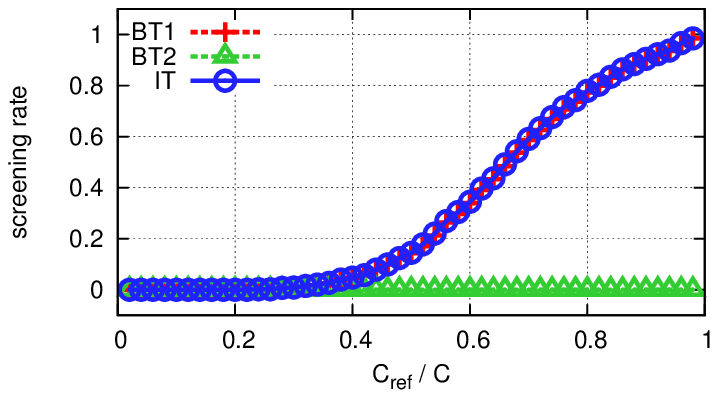} &
   \includegraphics[width = 0.225\textwidth]{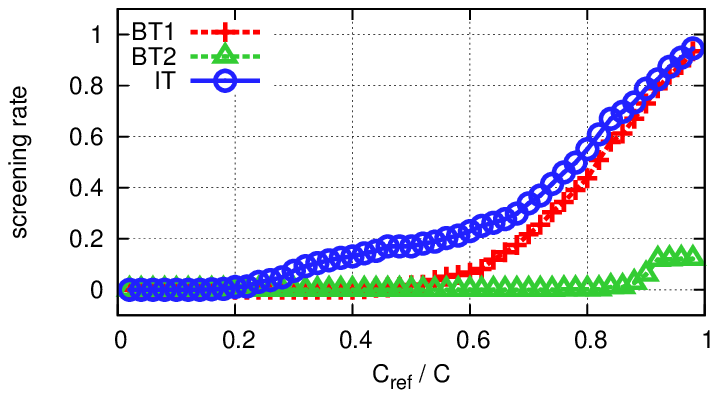} &
   \includegraphics[width = 0.225\textwidth]{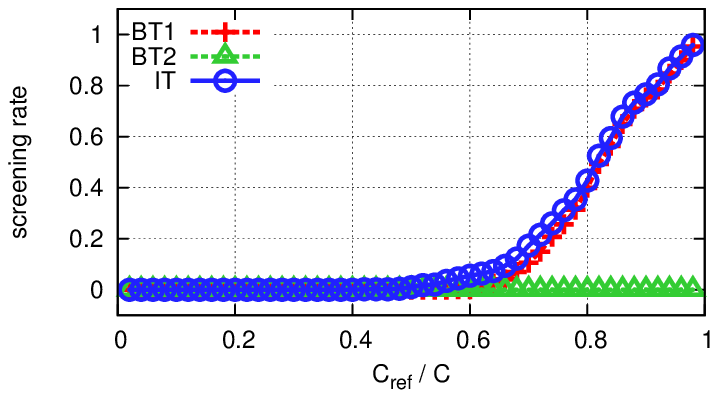} &
   \includegraphics[width = 0.225\textwidth]{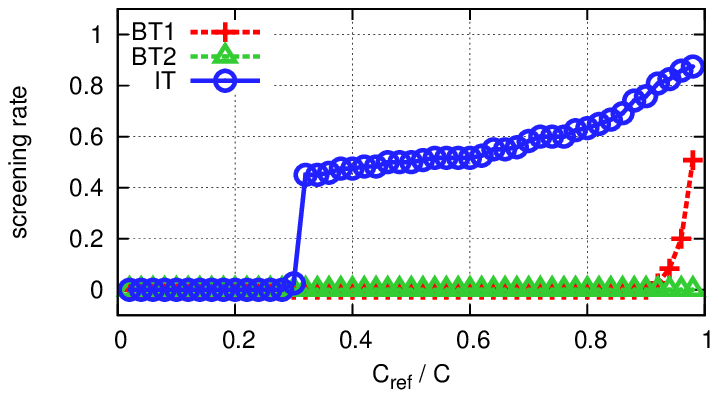} 
   \vspace*{-5mm}
   \\
   {\footnotesize D01, Linear} &
   {\footnotesize D01, RBF ($\gamma = 0.1/d$)} &      
   {\footnotesize D01, RBF ($\gamma = 1/d$)} &
   {\footnotesize D01, RBF ($\gamma = 10/d$)}
   \\
   \includegraphics[width = 0.225\textwidth]{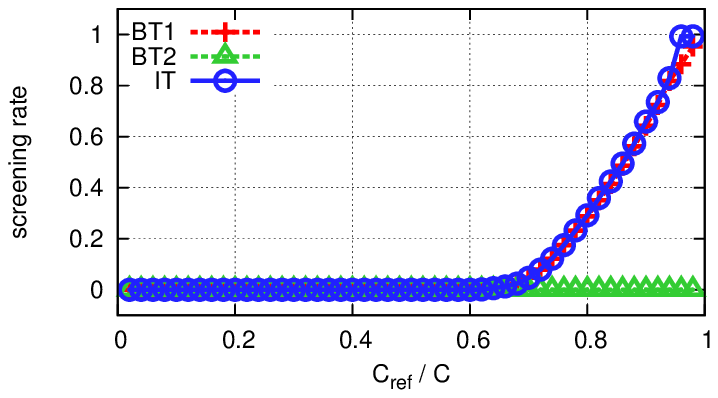} &
   \includegraphics[width = 0.225\textwidth]{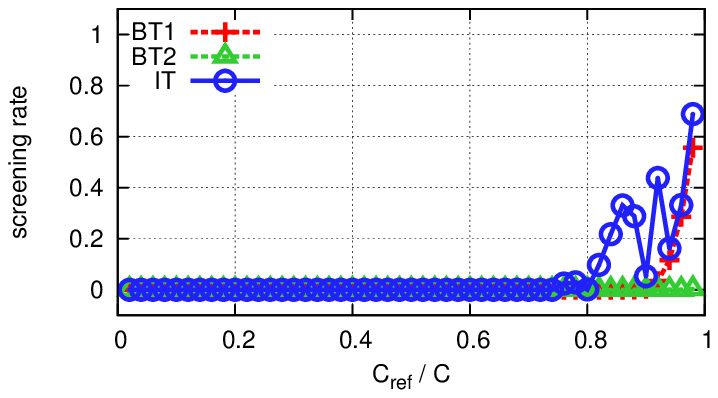} &
   \includegraphics[width = 0.225\textwidth]{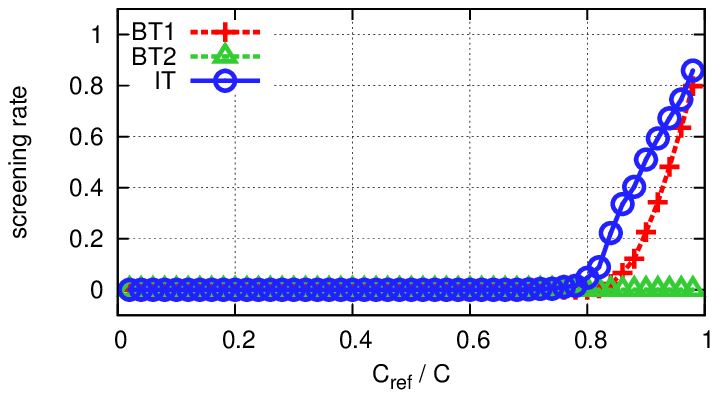} &
   \includegraphics[width = 0.225\textwidth]{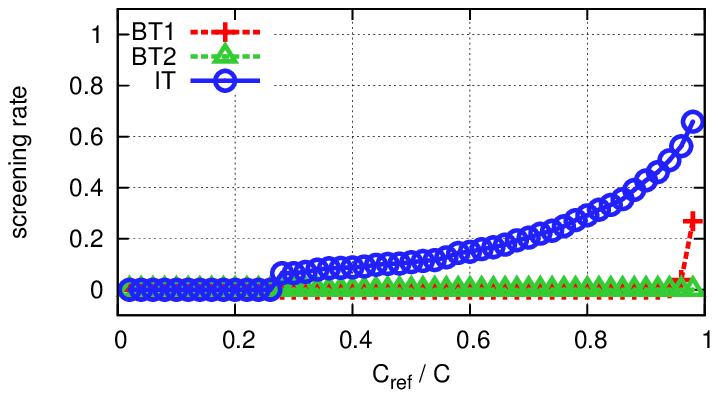} 
   \vspace*{-5mm}
   \\
   {\footnotesize D02, Linear} &
   {\footnotesize D02, RBF ($\gamma = 0.1/d$)} &      
   {\footnotesize D02, RBF ($\gamma = 1/d$)} &
   {\footnotesize D02, RBF ($\gamma = 10/d$)}
   \\
   \includegraphics[width = 0.225\textwidth]{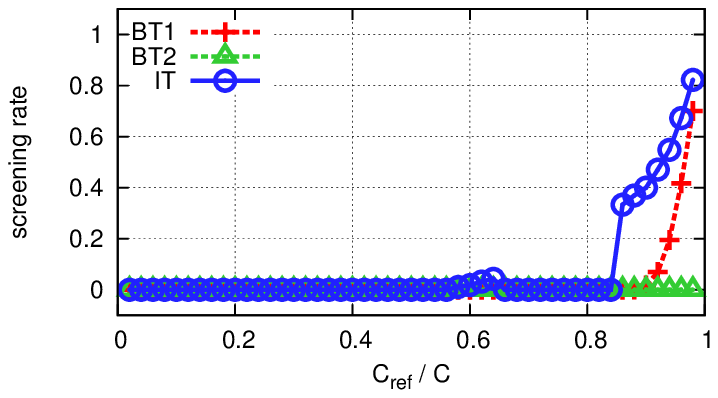} &
   \includegraphics[width = 0.225\textwidth]{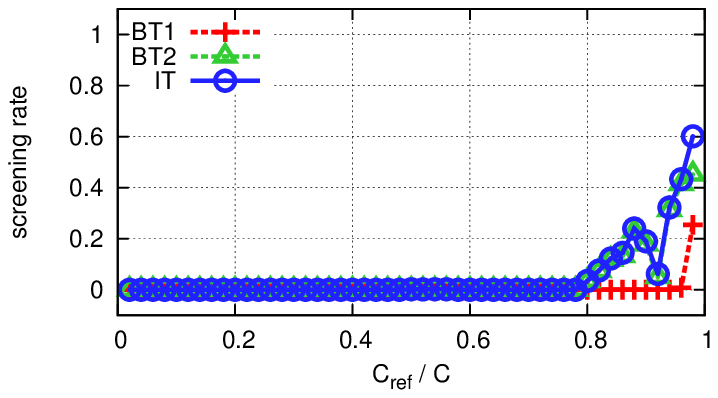} &
   \includegraphics[width = 0.225\textwidth]{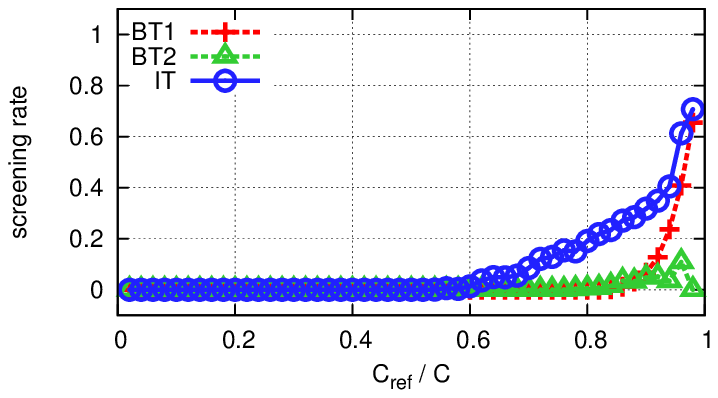} &
   \includegraphics[width = 0.225\textwidth]{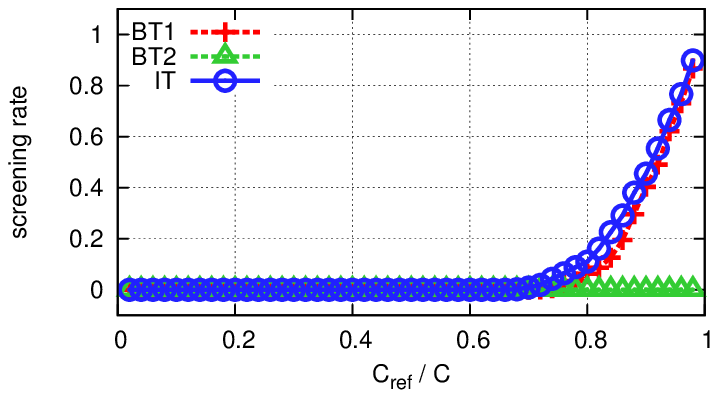} 
   \vspace*{-5mm}
   \\
   {\footnotesize D03, Linear} &
   {\footnotesize D03, RBF ($\gamma = 0.1/d$)} &      
   {\footnotesize D03, RBF ($\gamma = 1/d$)} &
   {\footnotesize D03, RBF ($\gamma = 10/d$)}
   \\
   \includegraphics[width = 0.225\textwidth]{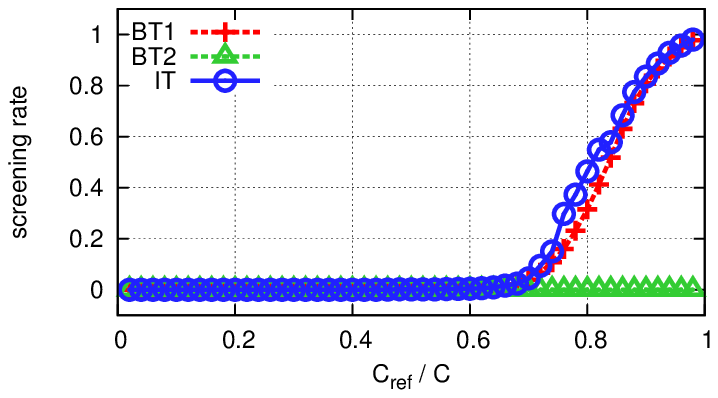} &
   \includegraphics[width = 0.225\textwidth]{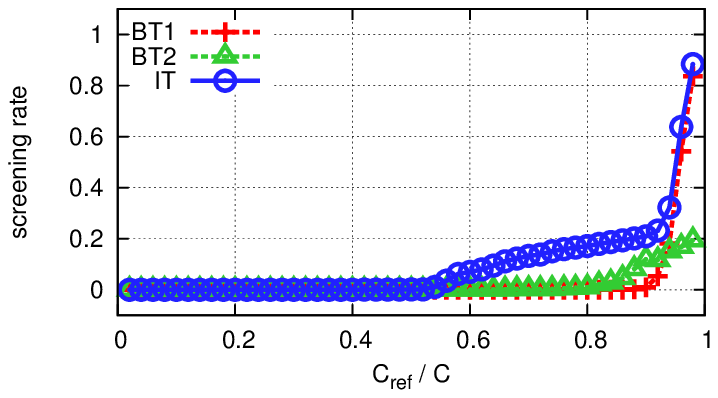} &
   \includegraphics[width = 0.225\textwidth]{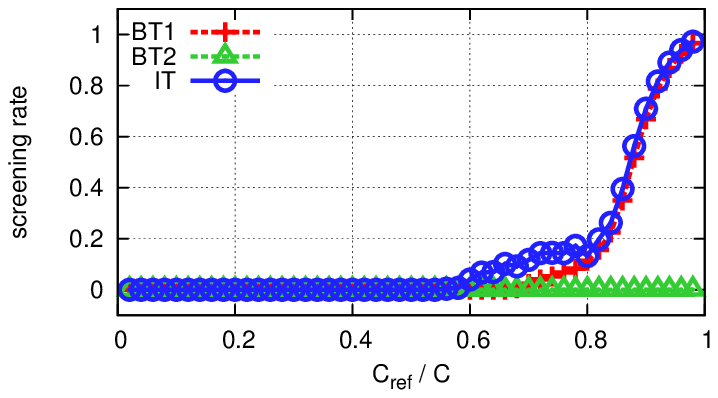} &
   \includegraphics[width = 0.225\textwidth]{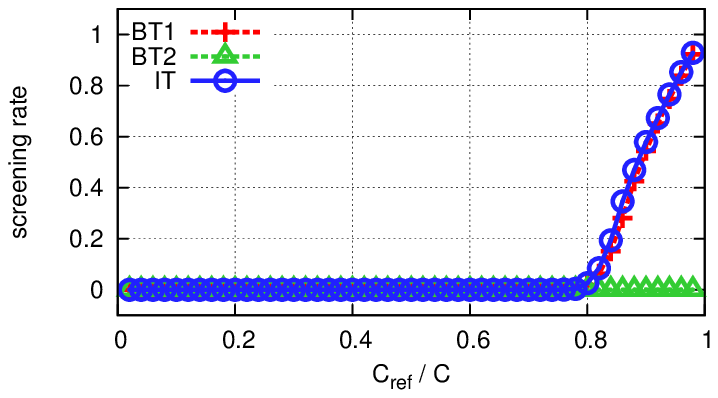} 
   \vspace*{-5mm}
   \\
   {\footnotesize D04, Linear} &
   {\footnotesize D04, RBF ($\gamma = 0.1/d$)} &      
   {\footnotesize D04, RBF ($\gamma = 1/d$)} &
   {\footnotesize D04, RBF ($\gamma = 10/d$)}
  \end{tabular}
  \caption{The screening rates of the three proposed safe screening tests
  BT1 (red), BT2 (green) and IT (blue).}
  \label{fig:screening.rate}
 \end{center}
\end{figure}
 
% 
% sec5.3
% 
\subsection{Computation time}
\label{subsec:computation.time}
We investigate how much the computational cost of the entire SVM
training process can be reduced by safe sample screening.
As the state-of-the-art SVM solvers,
we used 
{\libsvm}
\cite{Chang11a}
and 
{\liblinear}
\cite{Fan08b} 
for nonlinear
and
linear kernel cases,
respectively\footnote{
In this paper,
we only study exact batch SVM solvers,
and do not consider online or sampling-based approximate solvers such as 
\cite{Crammer06a,ShalevShwartz07a,Hazan11a}.
}.
Many SVM solvers %(including {\libsvm} and {\liblinear})
use
\emph{non-safe} sample screening heuristics 
in their inner loops.
The common basic idea in these heuristic approaches is to
\emph{predict}
which sample turns out to be SV or non-SV
(prediction step), 
and to solve a smaller optimization problem
defined only with the subset of the samples
predicted as SVs (optimization step).
These two steps must be repeated until all the optimality conditions in \eq{eq:SVM.opt.} are satisfied 
because the prediction step in these heuristic approaches is \emph{not safe}.
In {\libsvm} and {\liblinear}, such a heuristic is called \emph{shrinking}\footnote{
It is interesting to note that
shrinking algorithms
in
{\libsvm} and {\liblinear}
make their decisions 
based only on the (signed) margin
$y_i f(\bm x_i)$, i.e., 
if it is greater or smaller than a certain threshold,
the corresponding sample is predicted as a member of
$\cR$
or $\cL$,
respectively.
On the other hand, 
the decisions made by our safe sample screening methods 
do not solely depend on 
$y_i f(\bm x_i)$,
but also on the other quantities obtained from the reference solution
(see \figurename \ref{fig:toy.example} for example).
}.

%In our experiments,
We compared the total computational costs of the following six approaches: 
\begin{itemize}
 \item Full-sample training ({\bf Full}),

 \item Shrinking ({\bf Shrink}),

 \item Ball Test 1 ({\bf BT1}),

 \item Shrinking + Ball Test 1 ({\bf Shrink+BT1}).

 \item Intersection Test ({\bf IT}),

 \item Shrinking + Intersection Test ({\bf Shrink+IT}).
\end{itemize}
%
%\vspace*{1mm}
%Full-sample training ({Full}),
%
%\vspace*{1mm}
%({b}) Shrinking ({Shrink}),
%
%\vspace*{1mm}
%({c}) Ball Test 1 ({BT1}),
%
%\vspace*{1mm}
%({d}) Shrinking + Ball Test 1 {(Shrink+BT1)}.
%
%\vspace*{1mm}
%({e}) Intersection Test ({IT}),
%
%\vspace*{1mm}
%({f}) Shrinking + Intersection Test {({\bf Shrink+IT})}.

\vspace*{1mm}
\noindent
In {\bf Full} and {\bf Shrink},
we used 
{\libsvm} or {\liblinear}
with
and
without shrinking option,
respectively. 
In {\bf BT1} and {\bf Shrink+BT1},
we first screened out a subset of the samples by Ball Test 1, 
and the rest of the samples were fed into 
{\libsvm} or {\liblinear}
to solve the smaller optimization problem 
with 
and
without shrinking option,
respectively. 
In {\bf IT} and {\bf Shrink+IT}, 
we used Intersection Test for safe sample screening.

% -------------------------
% sec5.3.1
% -------------------------
\subsubsection{Single SVM training}
\label{subsubsec:one.model.SVM}
First,
we compared the computational costs 
of training a single linear SVM for the large data sets
($n > 50,000$).
Here,
our task was to find the optimal solution at the regularization parameter
$C = C_{\rm ref}/0.9$ 
using the reference solution at
$C_{\rm ref} = 500C_{\rm min}$.

Table \ref{tbl:one.model.SVM}
shows 
the average computational costs of 5 runs.
%
%Compared with the naive full-sample training
%({\bf Full}),
%both of 
%%shrinking
%({\bf Shrink})
%and 
%screening 
%({\bf BT1} and {\bf IT})
%worked reasonably well by itself.
%
%Comparing shrinking and screening, 
%the latter was always better than the former in all the cases. 
%
The best performance was obtained in all the setups 
when both shrinking and IT screening are simultaneously used
({\bf Shrink+IT}).
{\bf Shrink+BT1}
also performed well,
but it was consistently outperformed by 
{\bf Shrink+IT}. 
%
%The performances of Intersection Test 
%was consistently better than the counterparts of Ball Test 1.

%
% tab3
%
\begin{table}[h]
\caption{The computation time [sec] for training a single SVM.}
\label{tbl:one.model.SVM}
\begin{center}
\hspace*{-4em}
\begin{tabular}{c||c|c||c|c|c|c||c|c|c|c}
 & \multicolumn{2}{c||}{{LIBLINEAR}}
 & \multicolumn{6}{c}{Safe Sample Screening} \\ \cline{2-11} 
 Data set	& {\bf Full} & {\bf Shrink} & {\bf BT1} & {\bf Shrink+BT1} & {\bf Rule} & {\bf Rate} & {\bf IT} & {\bf Shrink+IT} & {\bf Rule} & {\bf Rate} 
\\
\hline \hline
D09	& 98.2	& 2.57	& 95.1	& 2.21	& 0.0022	& 0.178	& 47.3	& {\bf 1.21}	& 0.0214	& 0.51
\\ 
%
%D10	& 	& 5.09	& 	& {\bf 4.84}	& 0.0111984	& 0.186015	& 	& 5.23	& 0.0609908	& 0.253747
%\\ 
%
%D11	& 	& 	& 	& {\bf 199}	& 0.027996	& 0.666115	& 	& 204	& 0.942857	& 0.671094 
%\\ 
%
D10	& 1881	& 327	& 1690	& 247	& 0.0514	& 0.108	& 1575	& {\bf 228}	& 2.24	& 0.125
\\ 
D11	& 2801	& 115	& 2699	& 97.2	& 0.203	& 0.136	& 2757	& {\bf 88.1}	& 2.78	& 0.136
\\ 
D12	& 16875	& 4558	& 7170	& 4028	& 0.432	& 0.138	& 12002	& {\bf 3293}	& 5.39	& 0.139
\end{tabular} 

\vspace*{2.5mm}

The computation time of the best approach in each setup is written in boldface.
{\bf Rule} and {\bf Rate} indicate the computation time and 
the screening rate of the each rules, respectively.

\end{center}
\end{table}

% -------------------------
% sec5.3.2
% -------------------------
\subsubsection{Regularization path}
\label{subsubsec:eps.app.path}
As described in
\S \ref{subsec:regularization.path},
safe sample screening is especially useful
in regularization path computation scenario.
When we compute an SVM regularization path
for an increasing sequence of the regularization parameters 
$C_1 < \ldots < C_T$,
the previous optimal solution can be used as the reference solution.
We used a recently proposed 
\emph{$\eps$-approximation path ($\eps$-path)}
algorithm
\cite{Giesen12a,Giesen12b}
for setting a practically meaningful sequence of regularization parameters.
The detail $\eps$-approximation path procedure is described in Appendix
\ref{app:eps-path}.

%
% tab4
%
\begin{table}[p]
\caption{The computation time [sec] for computing regularization path.}
\label{tab:cost.path}
 \begin{center}
\begin{tabular}{c|c||c|c||c|c|c|c}
 & & \multicolumn{2}{c||}{{LIBSVM or LIBLINEAR}}
 &	 \multicolumn{4}{c}{Safe Sample Screening} \\ \cline{3-8} 
 Data set	& Kernel	& {\bf Full} & {\bf Shrink} & {\bf BT1} & {\bf Shrink+BT1} &  {\bf IT} & {\bf Shrink+IT} 
\\
\hline \hline
%%%%%%%%%%%%%%%%%%%%%%%%%%%%%%%%%%%%%%%%%%%%%%%%%%%%%%%%%%%%%%%%%%%%%%%%%%%
% B.C.D
& Linear		& 389	& 35.2	& 174	& {\bf 34.8}	& 177	& 34.8	
\\ \cline{2-8}
D01	
& RBF($0.1/d$)	& 43.8	& 4.51	& 9.08	& {\bf 2.8}	& 8.48	& 2.87
\\ 
& RBF($1/d$)	& 2.73	& 0.68	& 0.435	& 0.295	& 0.464	& {\bf 0.294}
\\ 
& RBF($10/d$)	& 0.73	& 0.4	& 0.312	& 0.221	& 0.266	& {\bf 0.213}
\\ \hline
%%%%%%%%%%%%%%%%%%%%%%%%%%%%%%%%%%%%%%%%%%%%%%%%%%%%%%%%%%%%%%%%%%%%%%%%%%%
% dna
& Linear		& 67	& 9.09	& 13.6	& {\bf 8.05}	& 13.4	& 8.14	
\\ \cline{2-8}
D02
& RBF($0.1/d$)	& 298	& 106	& 253	& 87.7	& 242	& {\bf 80.7}
\\ 
& RBF($1/d$)	& 13.9	& 5.27	& 7.14	& {\bf 2.5}	& 7.03	& 2.62
\\ 
& RBF($10/d$)	& 4.98	& 2.68	& 3.18	& 1.96	& 2.71	& {\bf 1.82}
\\ \hline
%%%%%%%%%%%%%%%%%%%%%%%%%%%%%%%%%%%%%%%%%%%%%%%%%%%%%%%%%%%%%%%%%%%%%%%%%%%
% DIGIT1
& Linear		& 369	& 59.3	& 221	& {\bf 56.7}	& 167	& 56.9	
\\ \cline{2-8}
D03
& RBF($0.1/d$)	& 938	& 261	& 928	& 262	& 741	& {\bf 203}
\\ 
& RBF($1/d$)	& 94.3	& 27.3	& 70.9	& 19.4	& 60.7	& {\bf 16.8}
\\ 
& RBF($10/d$)	& 6.93	& 2.71	& 2.92	& {\bf 0.77}	& 2.45	& 0.794
\\ \hline
%%%%%%%%%%%%%%%%%%%%%%%%%%%%%%%%%%%%%%%%%%%%%%%%%%%%%%%%%%%%%%%%%%%%%%%%%%%
% satimage
& Linear		& 3435	& 33.7	& 3256	& {\bf 33.2}	& 3248	& 33.2	
\\ \cline{2-8}
D04	
& RBF($0.1/d$)	& 1365	& 565	& 1325	& 547	& 1178	& {\bf 488}
\\ 
& RBF($1/d$)	& 635	& 218	& 392	& 129	& 277	& {\bf 88.7}
\\ 
& RBF($10/d$)	& 31	& 20.4	& 3.89	& {\bf 1.5}	& 3.87	& 1.68
\\ \hline
%%%%%%%%%%%%%%%%%%%%%%%%%%%%%%%%%%%%%%%%%%%%%%%%%%%%%%%%%%%%%%%%%%%%%%%%%%%
% gisette
& Linear		& 1532	& 350	& 894	& {\bf 318}	& 899	& 329	
\\ \cline{2-8}
D05	
& RBF($0.1/d$)	& 375	& 143	& 365	& 132	& 296	& {\bf 103}
\\ 
& RBF($1/d$)	& 63.9	& 30.1	& 33.4	& 13.5	& 25.4	& {\bf 10.2}
\\ 
& RBF($10/d$)	& 34.3	& 20.7	& 27.8	& 16.8	& 24.9	& {\bf 15.9}
\\ \hline
%%%%%%%%%%%%%%%%%%%%%%%%%%%%%%%%%%%%%%%%%%%%%%%%%%%%%%%%%%%%%%%%%%%%%%%%%%%
% mushrooms
& Linear		& 19.8	& 2.64	& 8.12	& 2.08	& 8.57	& {\bf 2.03}	
\\ \cline{2-8}
D06	
& RBF($0.1/d$)	& 1938	& 618	& 1838	& 572	& 1395	& {\bf 423}
\\ 
& RBF($1/d$)	& 239	& 103	& 164	& 62.3	& 134	& {\bf 50.6}
\\ 
& RBF($10/d$)	& 94.3	& 56.3	& 70.5	& 44.2	& 66.2	& {\bf 40.9}
\\ \hline
%%%%%%%%%%%%%%%%%%%%%%%%%%%%%%%%%%%%%%%%%%%%%%%%%%%%%%%%%%%%%%%%%%%%%%%%%%%
% news20
& Linear		& 2619	& {\bf 1665}	& 2495	& 1697	& 2427	& 1769	
\\ \cline{2-8}
D07
& RBF($0.1/d$)	& 10358	& 5565	& 10239	& {\bf 5493}	& 10245	& 5770
\\ 
& RBF($1/d$)	& 33960	& 12797	& 34019	& 12918	& 30373	& {\bf 10152}
\\ 
& RBF($10/d$)	& 270984	& 67348	& 270313	& 67062	& 264433	& {\bf 56427}
\\ \hline
%%%%%%%%%%%%%%%%%%%%%%%%%%%%%%%%%%%%%%%%%%%%%%%%%%%%%%%%%%%%%%%%%%%%%%%%%%%
% shuttle
& Linear		& 37135	& 67	& 35945	& {\bf 63.6}	& 36386	& 67.8	
\\ \cline{2-8}
D08	
& RBF($0.1/d$)	& 278232	& 63192	& 275688	& 63608	& 253219	& {\bf 51932}
\\ 
& RBF($1/d$)	& 214165	& 60608	& 203155	& 56161	& 180839	& {\bf 48867}
\\ 
& RBF($10/d$)	& 167690	& 54364	& 129490	& 45644	& 125675	& {\bf 44463}
\\ \hline
%%%%%%%%%%%%%%%%%%%%%%%%%%%%%%%%%%%%%%%%%%%%%%%%%%%%%%%%%%%%%%%%%%%%%%%%%%%
\end{tabular} 

\vspace*{2.5mm}

The computation time of the best approach in each setup is written in boldface. 

\end{center}

\end{table}

% Fig5
\begin{figure}[t]
 \begin{center}
  \begin{tabular}{cccc}
                                                                                                
   \includegraphics[width = 0.225\textwidth]{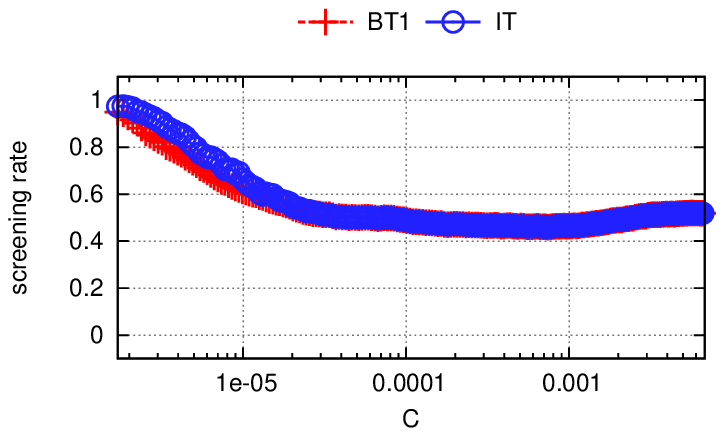} &
   \includegraphics[width = 0.225\textwidth]{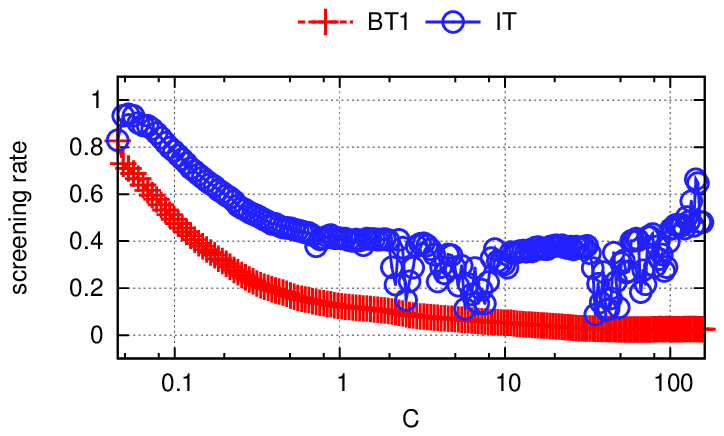} &
   \includegraphics[width = 0.225\textwidth]{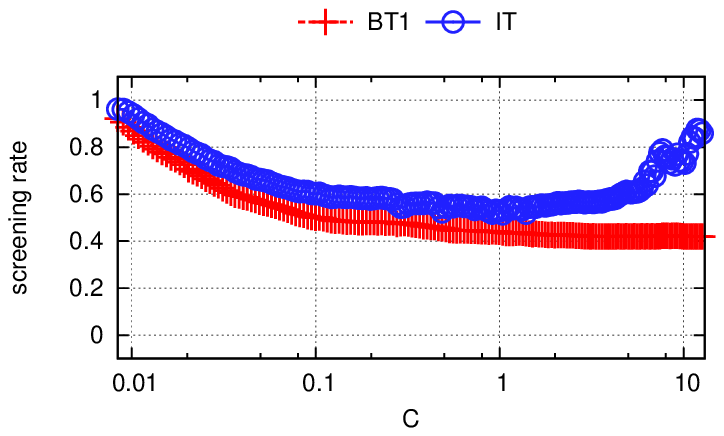} &
   \includegraphics[width = 0.225\textwidth]{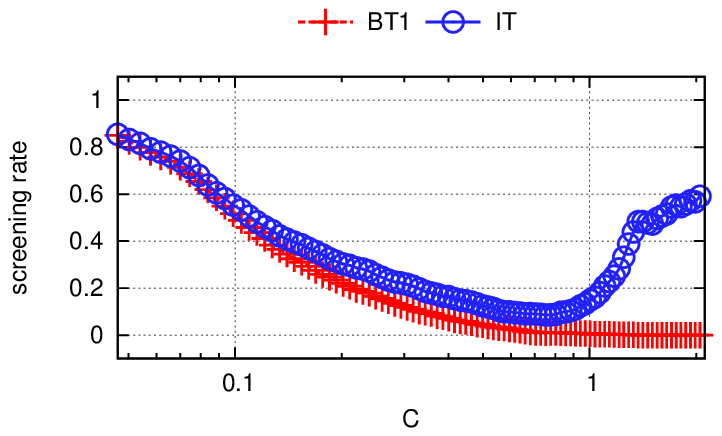} 
   \vspace*{-5mm}
   \\
   {\footnotesize D05, Linear} &
   {\footnotesize D05, RBF ($\gamma = 0.1/d$)} &      
   {\footnotesize D05, RBF ($\gamma = 1/d$)} &
   {\footnotesize D05, RBF ($\gamma = 10/d$) }
   \\
   \includegraphics[width = 0.225\textwidth]{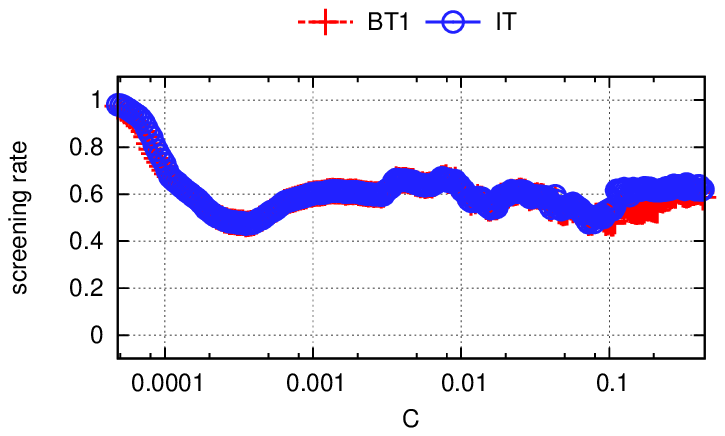} &
   \includegraphics[width = 0.225\textwidth]{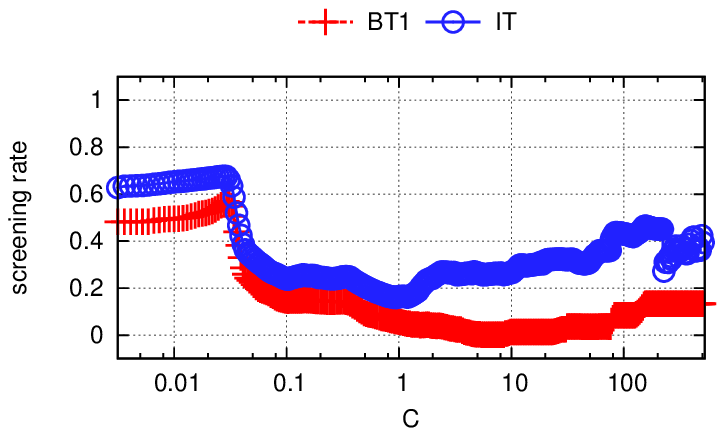} &
   \includegraphics[width = 0.225\textwidth]{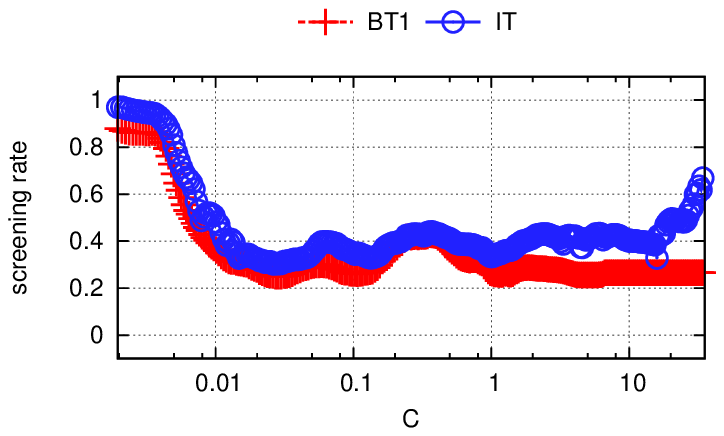} &
   \includegraphics[width = 0.225\textwidth]{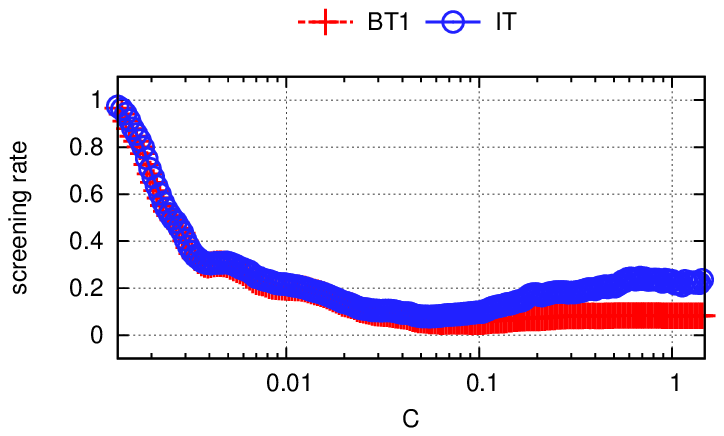} 
   \vspace*{-5mm}
   \\
   {\footnotesize D06, Linear} &
   {\footnotesize D06, RBF ($\gamma = 0.1/d$)} &      
   {\footnotesize D06, RBF ($\gamma = 1/d$)} &
   {\footnotesize D06, RBF ($\gamma = 10/d$) }
   \\
   \includegraphics[width = 0.225\textwidth]{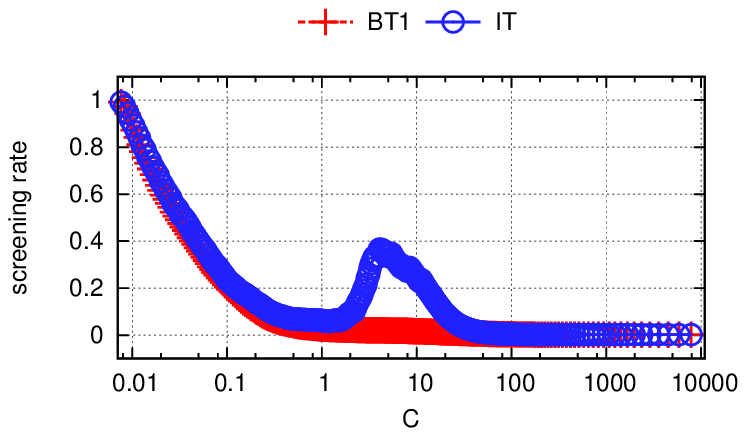} &
   \includegraphics[width = 0.225\textwidth]{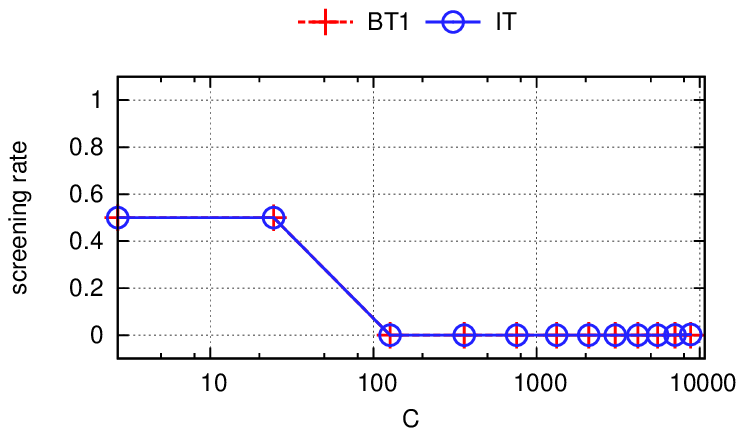} &
   \includegraphics[width = 0.225\textwidth]{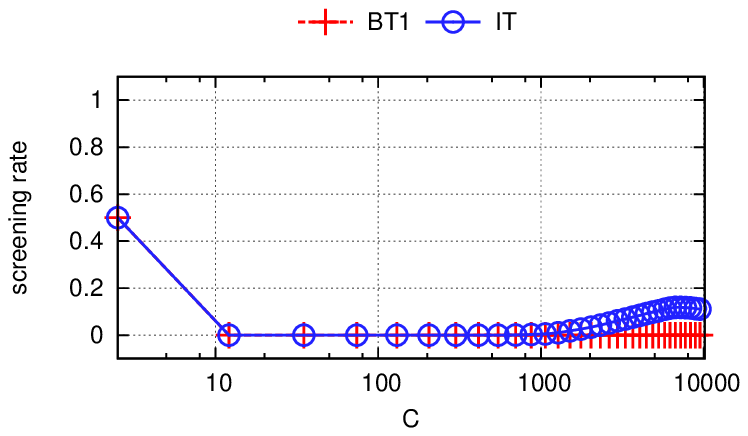} &
   \includegraphics[width = 0.225\textwidth]{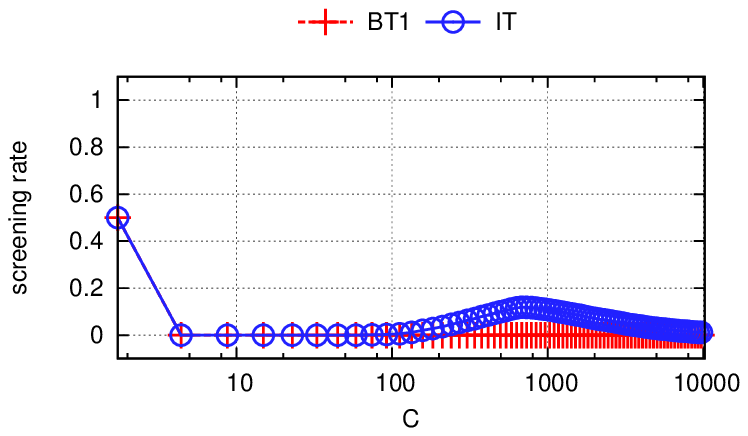} 
   \vspace*{-5mm}
   \\
   {\footnotesize D07, Linear} &
   {\footnotesize D07, RBF ($\gamma = 0.1/d$)} &      
   {\footnotesize D07, RBF ($\gamma = 1/d$)} &
   {\footnotesize D07, RBF ($\gamma = 10/d$) }
   \\
   \includegraphics[width = 0.225\textwidth]{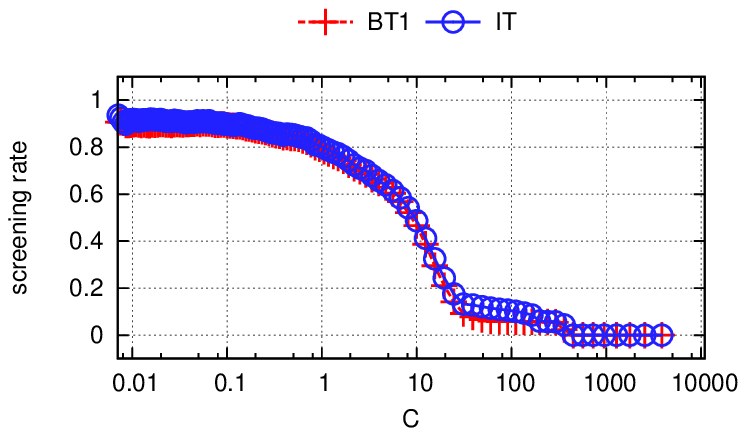} &
   \includegraphics[width = 0.225\textwidth]{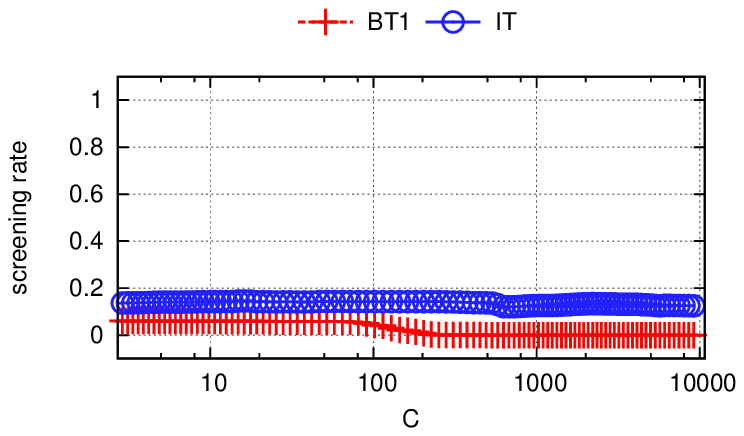} &
   \includegraphics[width = 0.225\textwidth]{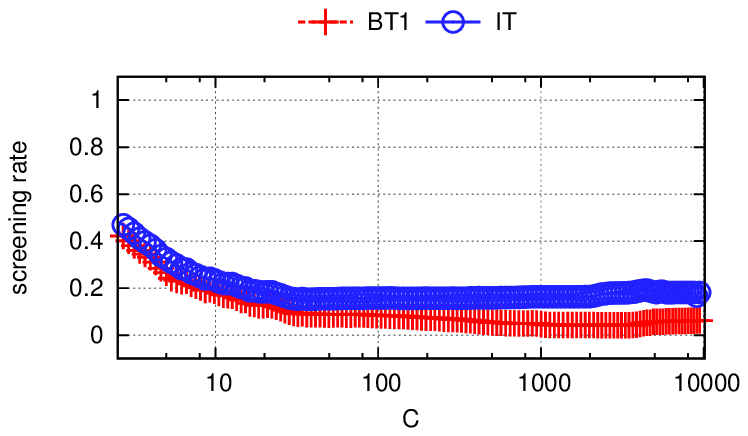} &
   \includegraphics[width = 0.225\textwidth]{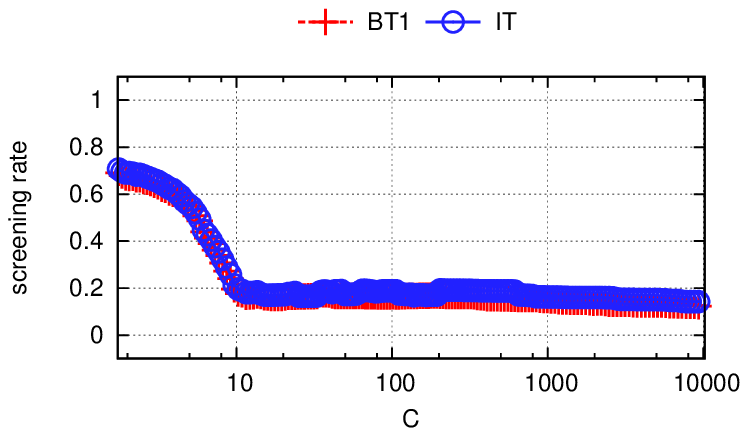} 
   \vspace*{-5mm}
   \\
   {\footnotesize D08, Linear} &
   {\footnotesize D08, RBF ($\gamma = 0.1/d$)} &      
   {\footnotesize D08, RBF ($\gamma = 1/d$)} &
   {\footnotesize D08, RBF ($\gamma = 10/d$) }
  \end{tabular}
  \caption{The screening rate in regularization path computation scenario for BT1 (red) and IT (blue).}
  \label{fig:eps.app}
 \end{center}
\end{figure}

In this scenario, 
we used the small and the medium data sets
($n \le 50,000$).
The largest regularization parameter was set as 
$C_T = 10^4$.
%for small data sets ($n \le 10,000$) 
%and 
%$C_T = 10^4$ 
%for medium data sets ($10,000 < n \le 50,000$). 
%
 We used linear kernel and RBF kernel
 $K(\bm x, \bm x^\prime) = \exp(-\gamma \|\bm x - \bm x^\prime\|^2)$
 with 
 $\gamma \in \{0.1/d, 1/d, 10/d\}$.
In all the six approaches,
we used the cache value and warm-start approach 
as described in
\S \ref{subsec:regularization.path}.
Table \ref{tab:cost.path}
summarizes the total computation time of the six approaches, 
and 
\figurename \ref{fig:eps.app}
shows how screening rates change with $C$ in each data set
(due to the space limitation, we only show the results on four medium data sets in \figurename \ref{fig:eps.app}).

Note first that shrinking heuristic was very helpful, and safe
sample screening alone
({\bf BT1} and {\bf IT}) 
was not as effective as shrinking. 
However, except one setup (D07, Linear), simultaneously using shrinking
and safe sample screening worked better than using shrinking alone.
As we discuss in \S \ref{subsec:complexity}, the rule evaluation cost of
BT1 is cheaper than that of IT.
Therefore, if the screening rates of these two tests are same, the
former is slightly faster than the latter.
In Table \ref{tab:cost.path},
we see that
{\bf Shrink+BT1}
was a little faster than
{\bf Shrink+IT}
in several setups. 
We conjecture that those small differences are due to the differences in the rule evaluation costs.
In the remaining setups, 
{\bf Shrink+IT}
was faster than 
{\bf Shrink+BT1}.
The differences tend to be small in the cases of linear kernel and RBF
kernel with relatively small $\gamma$.
On the other hand, significant improvements were sometimes observed
especially when RBF kernels with relatively large $\gamma$ is used.
In \figurename~\ref{fig:eps.app},
we confirmed that the screening rates of IT was never worse than BT1.
%The screening performances of these two tests were highly dependent
%on the data sets, the kernels, and the regularization parameters.

In summary,
the experimental results indicate that safe sample screening is often
helpful for reducing the computational cost of the state-of-the-art SVM
solvers.
Furthermore, 
Intersection Test seems to be the best safe sample screening method among those
we considered here.
%although BT1 could be a little bit faster when the
%problem is relatively easily solved (e.g., with linear kernel or RBF
%kernel with small $\gamma$). 

\section{Conclusion}
\label{sec:Conclusion}
In this paper,
we introduced safe sample screening approach 
that can safely identify and screen out a subset of the 
non-SVs prior to the training phase.
We believe that our contribution would be of great practical importance
in the current \emph{big data} era 
because it enables us to reduce the data size without sacrificing the optimality.
Our approach is quite general in the sense that it can be used together with any SVM solvers
as a preprocessing step for reducing the data set size.
The experimental results indicate that 
safe sample screening is not so harmful even when it cannot screen out any instances
because the rule evaluation costs are much smaller than that of SVM solvers.
Since the screening rates highly depend on the choice of the reference solution,
an important future work is to find a better reference solution.

\ifCLASSOPTIONcompsoc
  \section*{Acknowledgments}
\else
  \section*{Acknowledgment}
\fi

We thank Kohei Hatano and Masayuki Karasuyama for their furuitful comments.
We also thank Martin Jaggi for letting us know recent studies on approximate parametric programming.
IT thanks the supports from MEXT Kakenhi 23700165 and CREST, JST.

%\clearpage

\appendices
%
% app1
%
\section{Proofs}
\label{app:proofs}
% 
% Lemm1
% 
\begin{proof}[Proof of Lemma \ref{lemm:Ball.test}]
 The lower bound 
 $\ell_{[C]i}$
 is obtained as follows:
 \begin{eqnarray*}
 &&
 \min_{\bm w}
 ~
 y_i f(\bm x_i;\bm w)
 ~{\rm s.t.}~
 \| \bm w - \bm m \|^2 \le r^2
 =
 \min_{\bm w } \max_{\mu > 0}
 ~
 \bm z_i^\top \bm w + \mu ( \|\bm w - \bm m\|^2 - r^2) 
 \\
 &=&
\max_{\mu > 0}
~
(- \mu r^2)
+ \min_{\bm w}
~
( \mu \| \bm w - \bm m \|^2 + \bm z_i^\top \bm w )
=
\max_{\mu > 0} ~
L(\mu) \triangleq -\mu r^2 - \frac{\|\bm z_i\|^2}{4\mu} + \bm z_i^\top \bm m,
 \end{eqnarray*}
 where
 the Lagrange multiplier
 $\mu > 0$
 because the ball constraint is strictly active when the bound is attained.
 By solving
 $\partial L(\mu)/\partial \mu = 0$,
 %By letting the derivative of 
 %$L(\mu)$ 
 %w.r.t. 
 %$\mu$
 %be zero,
 the optimal Lagrange multiplier is given as 
 $\mu = \|\bm z_i\|/2r$. 
 Substituting this into $L(\mu)$, 
 we obtain 
 \begin{eqnarray*}
  \max_{\mu \ge 0} L(\mu) = \bm z_i^\top \bm m - r \|\bm z_i\|. 
 \end{eqnarray*}
 The upper bound $u_{[C]i}$ is obtained similarly. 
\end{proof}

%
% lemm2
%
%\subsection{The proof of Lemma~\ref{lemm:structural.svm.to.ball}}
%\label{subsec:lemm:structural.svm.to.ball}
\begin{proof}[Proof of Lemma~\ref{lemm:structural.svm.to.ball}]
 By substituting $\xi$ in the second inequality in
 \eq{eq:tilde.theta}
 into the first inequality,
 we immediately have
 $\| \bm w - \bm m\| \le r$. 
\end{proof}

% 
% Lemm3
% 
%\subsection{The proof of Lemma \ref{lemm:nc1}}
\begin{proof}[Proof of Lemma \ref{lemm:nc1}]
 From 
 Proposition 2.1.2 in
 \cite{Bertsekas99a},
 the optimal solution
 $(\bm w^*_{[C]}, \xi^*_{[C]})$
 and
 a feasible solution
 $(\tilde{\bm w}, \tilde{\xi})$
 satisfy the following relationship:
 \begin{eqnarray*}
  \nabla \cP_{[C]}(\bm w^*_{[C]}, \xi^*_{[C]})^\top
  \left(
  \mtx{c}{
  \tilde{\bm w} \\
  \tilde{\xi} \\
  }
  - 
 \mtx{c}{
 \bm w_{[C]}^* \\
 \xi_{[C]}^* \\
 }
 \right)
 = 
 \mtx{cc}{
 \bm w_{[C]}^{*\top} &
 C
 }
  \left(
  \mtx{c}{
  \tilde{\bm w} \\
  \tilde{\xi} \\
  }
  - 
 \mtx{c}{
 \bm w_{[C]}^* \\
 \xi_{[C]}^* \\
 }
 \right)
 \ge 0.
 \end{eqnarray*} 
\end{proof}

%
% Lemm:NC2
%
%\subsection{The proof of Lemma \ref{lemm:nc2}}
\begin{proof}[Proof of Lemma \ref{lemm:nc2}]
 From Proposition 2.1.2 in
 \cite{Bertsekas99a},
 the optimal solution
 $(\bm w^*_{[C_{\rm ref}]}, \xi^*_{[C_{\rm ref}]})$
 and
 a feasible solution
 $(\bm w^*_{[C]}, \xi^*_{[C]})$
 satisfy the following relationship:
 \begin{eqnarray*}
  \nabla \cP_{[\check{C}]}(\bm w^*_{[\check{C}]}, \xi^*_{[\check{C}]})^\top
  \left(
  \mtx{c}{
 \bm w_{[C]}^* \\
 \xi_{[C]}^* \\
  }
  - 
 \mtx{c}{
 \bm w_{[\check{C}]}^* \\
 \xi_{[\check{C}]}^* \\
 }
 \right)
 = 
 \mtx{cc}{
 \bm w_{[\check{C}]}^{*\top} &
 \check{C}
 }
  \left(
  \mtx{c}{
 \bm w_{[C]}^* \\
 \xi_{[C]}^* \\
  }
  - 
 \mtx{c}{
 \bm w_{[\check{C}]}^* \\
  \xi_{[\check{C}]}^* \\
 }
 \right)
 \ge 0.
 \end{eqnarray*}
\end{proof}

%
% Lemm:NC3
%
% \subsection{The proof of Lemma \ref{lemm:nc3}}
\begin{proof}[Proof of Lemma \ref{lemm:nc3}]
 \eq{eq:nc3} is necessary for the optimal solution just because it is 
 one of the $2^n$ constraints in 
 \eq{eq:str.SVM}.
\end{proof} 

% 
% Theo8 (Theo:IT)
% 
%\subsection{The proof of Theorem \ref{theo:IT}}
%\label{app:IT}
\begin{proof}[Proof of Theorem \ref{theo:IT}]
First,
we prove the following lemma.

 \vspace*{2.5mm}

{\it 
 \begin{lemm}
 \label{lemm:IT.equivalence}
 Let
 $\overline{\bm w} \in \cF$
 be the optimal solution of 
 \begin{eqnarray}
  \label{eq:IT.opt.prob.1}
  \min_{\bm w} ~ \bm z_i^\top \bm w
  ~
  {\rm s.t.} ~ \| \bm w - \bm m_1 \|^2 \le r_1^2,  ~ \| \bm w - \bm m_2 \|^2 \le r_2^2,
 \end{eqnarray}
 and 
 $(\underline{\bm w}, \underline{\xi}) \in \cF \times \RR$
 be the optimal solution of 
 \begin{eqnarray}
  \label{eq:IT.opt.prob.2}
  \min_{\bm w, \xi} ~ \bm z_i^\top \bm w
  ~
  {\rm s.t.}
  ~
  \| \bm w \|^2 \le \xi,
  ~
  \xi \le 2 \bm m_1^\top \bm w + r_1^2 - \| \bm m_1 \|^2,
  ~
  \xi \le 2 \bm m_2^\top \bm w + r_2^2 - \| \bm m_2 \|^2.
 \end{eqnarray}
 Then, the two optimization problems 
 \eq{eq:IT.opt.prob.1}
 and
 \eq{eq:IT.opt.prob.2}
 are equivalent in the sense that 
 $\bm z_i^\top \overline{\bm w} = \bm z_i^\top \underline{\bm w}$.
 \end{lemm}
}
 \begin{proof}
 Let
 $\overline{\xi} \triangleq \| \overline{\bm w} \|^2$.
 Then,
 $(\overline{\bm w}, \overline{\xi})$
 is a feasible solution of 
 \eq{eq:IT.opt.prob.2}
 because
 \begin{eqnarray*}
  \| \overline{\bm w} - \bm m_1 \|^2 \le r_1^2 
  ~\Rightarrow~
  2 \bm m_1^\top \overline{\bm w} + r_1^2 - \| \bm m_1 \|^2 \ge \| \overline{\bm w}\|^2 = \overline{\xi},
  \\
  \| \overline{\bm w} - \bm m_2 \|^2 \le r_2^2 
  ~\Rightarrow~
  2 \bm m_2^\top \overline{\bm w} + r_2^2 - \| \bm m_2 \|^2 \ge \| \overline{\bm w}\|^2 = \overline{\xi}.
 \end{eqnarray*}
 On the other hand,
 $(\underline{\bm w}, \underline{\xi})$
 is a feasible solution of 
 \eq{eq:IT.opt.prob.1}
 because
 \begin{eqnarray*}
  \| \underline{\bm w} \|^2 \le \underline{\xi}
  \text{ and }
  \underline{\xi} \le 2 \bm m_1^\top \underline{\bm w} + r_1^2 - \| \bm m_1 \|^2
  ~\Rightarrow~  
  \| \underline{\bm w} - \bm m_1 \|^2 \le r_1^2,
  \\
  \| \underline{\bm w} \|^2 \le \underline{\xi}
  \text{ and }
  \underline{\xi} \le 2 \bm m_2^\top \underline{\bm w} + r_2^2 - \| \bm m_2 \|^2
  ~\Rightarrow~  
  \| \underline{\bm w} - \bm m_2 \|^2 \le r_2^2.
 \end{eqnarray*}
 These facts indicate that 
 $\bm z_i^\top \overline{\bm w} = \bm z_i^\top \underline{\bm w}$
 for arbitrary
 $\bm z_i \in \cF$. 
% we can conclude that
% $\overline{\bm w} = \underline{\bm w}$.
 \end{proof}

 \vspace*{2.5mm}
 
 %Let us call the two balls
 %$\Theta^{\rm (BT1)}_{[C]i}$
 %and
 %$\Theta^{\rm (BT2)}_{[C]i}$
 %as $\Theta^{\rm (BT1)}_{[C]i}$ and $\Theta^{\rm (BT2)}_{[C]i}$,
 %respectively.
 %
 We first note that at least one of the two balls $\Theta^{\rm (BT1)}_{[C]i}$ and $\Theta^{\rm (BT2)}_{[C]i}$ 
 are strictly active when the lower bound is attained.
 It means that we can only consider the following three cases:
 \begin{itemize}
  \item Case 1) $\Theta^{\rm (BT1)}_{[C]i}$ is active and $\Theta^{\rm (BT2)}_{[C]i}$ is inactive, 
  \item Case 2) $\Theta^{\rm (BT2)}_{[C]i}$ is active and $\Theta^{\rm (BT1)}_{[C]i}$ is inactive, and
  \item Case 3) Both $\Theta^{\rm (BT1)}_{[C]i}$ and $\Theta^{\rm (BT2)}_{[C]i}$ are active. 
 \end{itemize}

 From Lemma \ref{lemm:IT.equivalence},
 the lower bound
 $\ell^{\rm (IT)}_{[C]i}$
 is the solution of 
 \begin{eqnarray}
  \label{eq:IT.proof.equivalent.prob}
  \min_{\bm w, \xi}
  ~
  \bm z_i^\top \bm w
  ~
  {\rm s.t.}
  ~
  \| \bm w \|^2 \le \xi,
  ~
  \xi \le 2 \bm m_1^\top \bm w + r_1^2 - \| \bm m_1 \|^2,
  ~
  \xi \le 2 \bm m_2^\top \bm w + r_2^2 - \| \bm m_2 \|^2.
 \end{eqnarray}
 Introducing the Lagrange multipliers 
 $\mu, \nu_1, \nu_2 \in \RR_{+}$
 for the three constraints in 
 \eq{eq:IT.proof.equivalent.prob},
 we write the Lagrangian of the problem 
 \eq{eq:IT.proof.equivalent.prob}
 as
 $L(\bm w, \xi, \mu, \nu_1, \nu_2)$.
 From the stationary conditions,
 we have
 \begin{eqnarray}
  \label{eq:IT.proof.stationary.cond}
  \pd{L}{\bm w} = 0
  ~\Leftrightarrow~
  \bm w = \frac{1}{2 \mu} (2 \nu_1 \bm m_1 + 2 \nu_2 \bm m_2 - \bm z_i),
  ~
  \pd{L}{\xi} = 0
  ~\Leftrightarrow~
  \mu - \nu_1 - \nu_2 = 0.
 \end{eqnarray}
 where
 $\mu > 0$
 because at least one of the two balls $\Theta^{\rm (BT1)}_{[C]i}$ and $\Theta^{\rm (BT2)}_{[C]i}$ are strictly active. 
 
 {\bf Case 1})
 Let us first consider the case where $\Theta^{\rm (BT1)}_{[C]i}$ is active and $\Theta^{\rm (BT2)}_{[C]i}$ is inactive,
 i.e., 
 $\| \bm w - \bm m_1 \|^2 = r_1^2$
 and 
 $\| \bm w - \bm m_2 \|^2 < r_2^2$.
 Noting that 
 $\nu_2 = 0$,
 the latter can be rewritten as 
 \begin{eqnarray*}
  \| \bm w - \bm m_2 \|^2 < r_2^2
%  ~\Leftrightarrow~
%  \frac{\bm z_i^\top (\bm m_1 - \bm m_2)}{\|\bm z_i \| \|\bm m_1 - \bm m_2\|}
%  < \frac{\| \bm m_1 - \bm m_2 \| + 2 (r_1^2 - r_2^2)}{2 r_1}
  ~\Leftrightarrow~
  \frac{- \bm z_i^\top \bm \phi}{\|\bm z_i\| \|\bm \phi\|} < \frac{\zeta - \|\bm \phi\|}{r_1},
 \end{eqnarray*}
 where we have used the stationary condition in 
 \eq{eq:IT.proof.stationary.cond}.
 In this case, 
 it is clear that the lower bound is identical with that of BT1, i.e., 
 $\ell^{\rm {(IT)}}_{[C]i} = \ell^{\rm {(BT1)}}_{[C]i}$. 

 {\bf Case 2})
 Next, let us consider the case where $\Theta^{\rm (BT2)}_{[C]i}$ is active and $\Theta^{\rm (BT1)}_{[C]i}$ is inactive,
 i.e., 
 $\| \bm w - \bm m_2 \|^2 = r_2^2$
 and 
 $\| \bm w - \bm m_1 \|^2 < r_1^2$.
 In the same way as Case 1),
 the latter condition is rewritten as 
 \begin{eqnarray*}
  \| \bm w - \bm m_1 \|^2 < r_1^2
%  ~\Leftrightarrow~
%  \frac{\bm z_i^\top (\bm m_2 - \bm m_1)}{\|\bm z_i \| \|\bm m_2 - \bm m_1\|}
%  < \frac{\| \bm m_2 - \bm m_1 \| + 2 (r_2^2 - r_1^2)}{2 r_2}
  ~\Leftrightarrow~
  \frac{\zeta}{r_2} < \frac{- \bm z_i^\top \bm \phi}{\|\bm z_i\| \|\bm \phi\|}.
 \end{eqnarray*}
 In this case, 
 the lower bound of IT is identical with that of BT2, i.e., 
 $\ell^{\rm {(IT)}}_{[C]i} = \ell^{\rm {(BT2)}}_{[C]i}$. 

 {\bf Case 3})
 Finally, let us consider the remaining case where both of the two balls $\Theta^{\rm (BT1)}_{[C]i}$ and $\Theta^{\rm (BT2)}_{[C]i}$ are strictly active. 
 From the conditions of Case 1) and Case 2),
 the condition of Case 3) is written as 
\begin{eqnarray}
 \label{eq:proof6.case3.cond}
 \frac{\zeta - \|\bm \phi\|}{r_1}
 \le
 \frac{- \bm z_i^\top \bm \phi}{\|\bm z_i\| \|\bm \phi\|}
 \le
 \frac{\zeta}{r_2}.
\end{eqnarray}
 After plugging the stationary conditions 
 \eq{eq:IT.proof.stationary.cond} 
 into
 $L(\bm w, \xi, \mu, \nu_1, \nu_2)$, 
 the solution of the following linear system of equations
\begin{eqnarray*}
 \pd{L}{\mu} = 0, ~
 \pd{L}{\nu_1} = 0, ~
 \pd{L}{\nu_2} = 0, 
\end{eqnarray*} 
 are given as 
\begin{eqnarray}
 \label{eq:IT-proof.mu.nu1.nu2}
 \mu
 =
 \frac{1}{2\kappa}\sqrt{\|\bm z_i\|^2 - \frac{(\bm z_i^\top\bm \phi)^2}{\|\bm \phi\|^2}},
 ~
 \nu_1
 = 
 \mu \frac{\zeta}{\|\bm \phi\|} + \frac{\bm z_i^\top \bm \phi}{2\|\bm \phi\|^2},
 ~
 \nu_2
 = 
 \mu - \nu_1.
\end{eqnarray}
From
\eq{eq:proof6.case3.cond}, 
$\mu, \nu_1, \nu_2$
in 
\eq{eq:IT-proof.mu.nu1.nu2}
are shown to be non-negative,
meaning that 
\eq{eq:IT-proof.mu.nu1.nu2}
are the optimal Lagrange multipliers.
 By plugging these 
 $\mu, \nu_1, \nu_2$
 into
 $\bm w$ in 
 \eq{eq:IT.proof.stationary.cond},
the lower bound is obtained as
\begin{eqnarray*}
 \ell_i^{\rm (IT)}
   = 
   \bm z_i^\top \bm \psi - \kappa \sqrt{\| \bm z_i \|^2 - \frac{(\bm z_i^\top \bm \phi)^2}{\| \bm \phi \|^2}}.
 \end{eqnarray*}

 By combining all the three cases above, 
 the lower bound 
 \eq{eq:it.test.low}
 is asserted. 
 The upper bound 
 \eq{eq:it.upp}
 can be similarly derived.
\end{proof}

% -------------------------
% Lemma9
% -------------------------
% \subsection{The proof of lemma \ref{lemm:C.min}}
%\label{app:C.min}
%
\begin{proof}[Proof of Lemma \ref{lemm:C.min}]
 It is suffice to show that
 $\bm \alpha = C \ones$
 satisfies the optimality condition for any 
 $C \in (0,C_{min}]$.
 Remembering that
 $f(\bm x_i) = \sum_{j \in \NN_n} \alpha_j y_j K(\bm x_i, \bm x_j) = C (\bm Q \ones)_i$, 
 we have 
\begin{eqnarray*}
%\label{eq:proof.lemm7}
\max_{i \in \NN_n} y_i f(\bm x_i) 
= 
\max_{i \in \NN_n} C (\bm Q \ones)_i
\le
C_{\rm min} \max_{i \in \NN_n} (\bm Q \ones)_i
= 
1. 
\end{eqnarray*}
Noting that 
positive semi-definiteness of the matrix 
$\bm Q$
indicates  
$\ones^\top \bm Q \ones \ge 0$,
the above inequality holds 
because at least one component of 
$\bm Q \ones$ 
must have nonnegative value.
It implies that all the $n$ samples are in
either 
$\cE$ or $\cL$,
where
$\alpha_i = C~\forall i \in \NN_n$
clearly satisfies the optimality.
\end{proof}

% ---------------------------------------------------------------------------
% 
% AppC
% 
\section{A comparison with the method in \cite{Ogawa13a}}
\label{app:v.s.ICML.ver.}

% --------------------------------------------------
% Fig6
% --------------------------------------------------
\begin{figure}[t]
\begin{center}
\begin{tabular}{cc}
\includegraphics[width=0.45\textwidth]{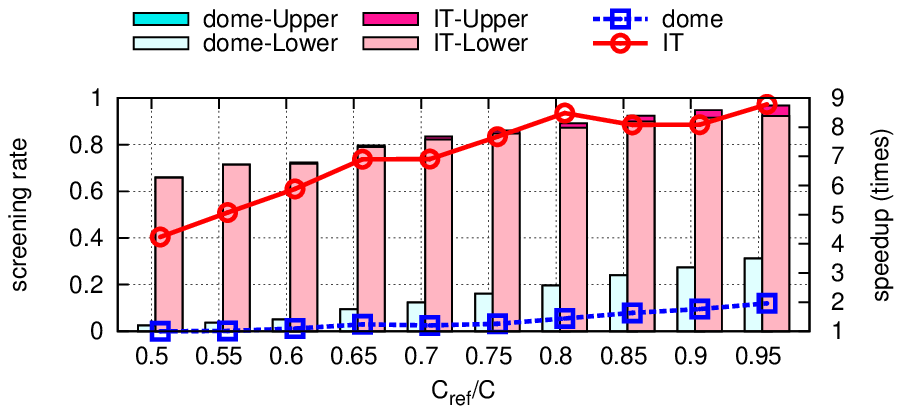} &
\includegraphics[width=0.45\textwidth]{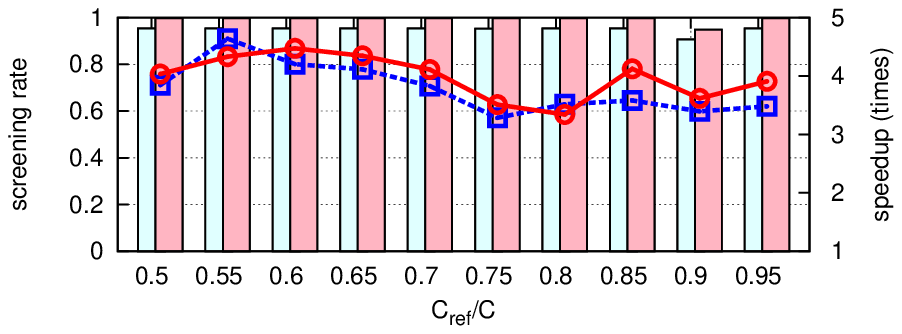}
\\
\small{{\bf B.C.D.} ($n=569,d=30$)} &
\small{{\bf PCMAC} ($n=1946,d=7511$)}
\\
\includegraphics[width=0.45\textwidth]{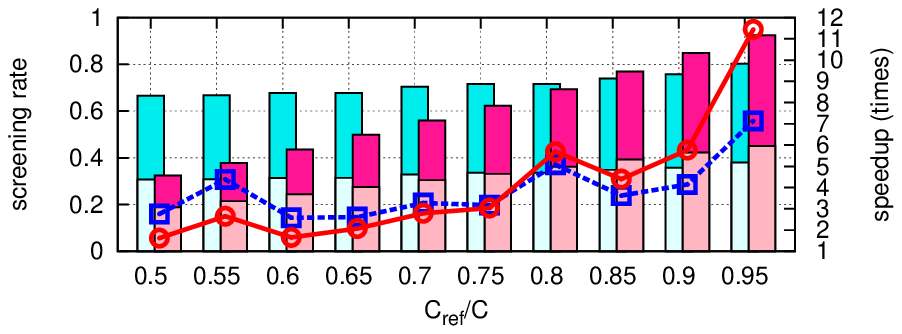} &
\includegraphics[width=0.45\textwidth]{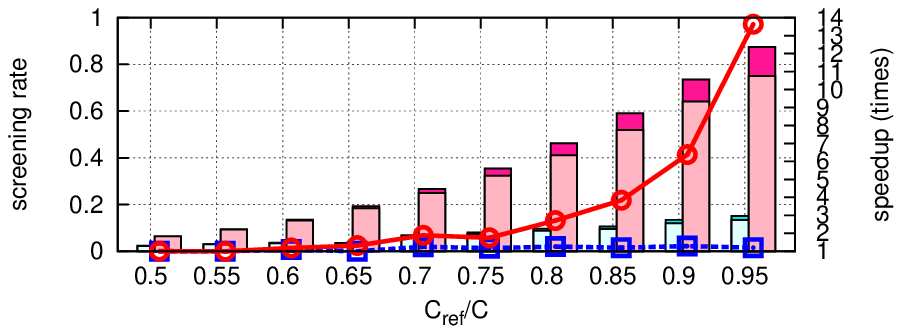}
\\
\small{{\bf MAGIC.} ($n=19020,d=10$)} &
\small{{\bf IJCNN1} ($n=19990,d=22$)}
\end{tabular}
%\end{tabular}
 \caption{
The comparison between
Intersection Test
and 
Dome Test
\cite{Ogawa13a}.
The red and blue bars (the left vertical axis)
indicate
the screening rates, i.e., 
the number of screened samples in 
$\cR$ and $\cL$
out of the total size
$|\cR| + |\cL|$. 
The red and blue lines (the right vertical axis)
show
the speedup improvement, 
where the baseline is naive full-sample training without any screening.}
\label{fig:v.s.ICML.ver.}
\end{center}
\end{figure}

We briefly describe the safe sample screening method
proposed in our preliminary conference paper
\cite{Ogawa13a},
which we call,
\emph{Dome Test (DT)}\footnote{
We call it as
\emph{Dome Test}
because the shape of the region
$\Theta$ 
looks like a dome
(see \cite{Ogawa13a} for details).}.
We discuss the difference among DT and IT, 
and 
compare their screening rates and computation times in simple numerical experiments. 
DT is summarized in the following theorem:

\begin{it}
 \begin{theo}[Dome Test]
\label{theo:dome.test}
% Consider the problem of training an SVM in the form of 
% \eq{eq:p.SVM.prob.}
% with the regularization parameter $C$. 
 %
 Consider two positive scalars
 $C_a < C_b$. 
 Then,
 for any
 $C \in [C_a, C_b]$, 
 the lower and the upper bounds of
 $y_i f(\bm x_i; \bm w^*_{[C]})$
 are given by 
 \begin{align*}
 \ell_{[C]i}^{\rm (DT)}
 \triangleq
 \min_{\bm w \in \Theta} y_i f(\bm x_i; \bm w) 
 = 
 \mycase{
-\sqrt{2 \gamma_b} \|\bm z_i\| &
 \text{if}~
{\scriptstyle\frac{-\bm z_i^\top\bm w_{[C_a]}^*}{\| \bm z_i \|}} \ge 
{\scriptstyle\frac{\gamma_a\sqrt{2}}{\sqrt{\gamma_b}}}
  \\
\bm z_i^\top\bm w_{[C_a]}^* 
- \sqrt{{\scriptstyle\frac{\gamma_b - \gamma_a}{\gamma_a}}(\gamma_a\|\bm z_i\|^2 - (\bm z_i^\top \bm w_{[C_a]}^*)^2 }
  &
  \text{otherwise}.
}
 \end{align*}
  and 
 \begin{align*}
&
u_{[C]i}^{\rm (DT)} \triangleq \max_{\bm w \in \Theta} y_i f(\bm x_i; \bm w) 
=
\mycase{
\sqrt{2 \gamma_b} \|\bm z_i\|
&
\text{if}~
{\scriptstyle\frac{\bm z_i^\top\bm w_{[C_a]}^*}{\| \bm z_i \|}} \ge 
{\scriptstyle\frac{\gamma_a\sqrt{2}}{\sqrt{\gamma_b}}}
\\
\bm z_i^\top\bm w_{[C_a]}^* 
+ \sqrt{{\scriptstyle\frac{\gamma_b - \gamma_a}{\gamma_a}}(\gamma_a\|\bm z_i\|^2 - (\bm z_i^\top \bm w_{[C_a]}^*)^2 }
&
\text{otherwise},
}
 \end{align*}
where 
$\gamma_a \triangleq \|\bm w_{[C_a]}^*\|^2$
and
$\gamma_b = \|\bm w_{[C_b]}^*\|^2$.
 \end{theo}
 \end{it}

\noindent
See 
\cite{Ogawa13a}
for the proof. 
A limitation of DT is that we need to know a feasible solution with a larger $C_b > C$
as well as the optimal solution with a smaller $C_a < C$
(remember that we only need the latter for BT1, BT2 and IT). 
As discussed in
\S 
\ref{subsec:regularization.path},
we usually train an SVM regularization path
from smaller $C$ to larger $C$ by using warm-start approach. 
Therefore,
it is sometimes computationally expensive to obtain a feasible solution with a larger $C_b > C$. 
In 
\cite{Ogawa13a},
we have used a bit tricky algorithm for obtaining such a feasible solution.

\figurename \ref{fig:v.s.ICML.ver.} 
shows the results of empirical comparison among DT and IT
on the four data sets used in
\cite{Ogawa13a}
with linear kernel
(CVX
\cite{cvx12a}
is used as the SVM solver in order to simply compare the effects of the screening performances).
Here,
we fixed
$C_{\rm ref} = C_a = 10^4 C_{\rm min}$
and varied
$C$
in the range of
$[0.5C_{\rm ref}, 0.95C_{\rm ref}]$.
For DT,
we assumed that the optimal solution with
$C_b = 1.3C$
can be used as a feasible solution
although it is a bit unfair setup for IT. 
We see that,
IT is clearly better in
\emph{B.C.D.}
and
\emph{IJCNN1},
comparable in
\emph{PCMAC}
and
slightly worse in
\emph{MAGIC}
data sets 
albeit a bit unfair setup for IT. 
The reason why
DT behaved poorly even when $C_{\rm ref}/C$ is close to 1
is that the lower and the upper bounds in DT
depends on the value
$(\gamma_b - \gamma_a)/\gamma_a$, 
and does not depend on $C$ itself. 
It means that,
when the range 
$[C_a,C_b]$ 
is somewhat large, 
the performance of DT deteriorate.

%
% AppB: Equivalence with DVI
%
\section{Equivalence between a special case of BT1 and the method in Wang et al. \cite{Wang13e}}
\label{app:equivalence.DVI}
When we use the reference solution
$\bm w^*_{[C_{\rm ref}]}$
as both of the feasible solution and the (different) optimal solution, 
the lower bound by BT1 is written as 
\begin{eqnarray*}
\ell^{\rm (BT1)}_{[C]i}
 =
 \frac{C + C_{\rm ref}}{2 C_{\rm ref}} \bm z_i^\top \bm w^*_{[C_{\rm ref}]}
 -
 \frac{C - C_{\rm ref}}{2 C_{\rm ref}} \| \bm w^*_{[C_{\rm ref}]} \| \| \bm z_i \|
\end{eqnarray*}
Using the relationships described in
\S \ref{subsec:kernelization},
the dual form of the lower bound is written as 
\begin{eqnarray}
 \label{eq:DVI.equiv.2}
\ell^{\rm (BT1)}_{[C]i}
 =
 \frac{C + C_{\rm ref}}{2 C_{\rm ref}} (\bm Q \bm \alpha^*_{[C_{\rm ref}]})_i
 -
 \frac{C - C_{\rm ref}}{2 C_{\rm ref}}
 \sqrt{\bm \alpha^{*\top}_{[C_{\rm ref}]} \bm Q \bm \alpha^*_{[C_{\rm ref}]} Q_{ii}}.
\end{eqnarray}
After transforming some variables, 
\eq{eq:DVI.equiv.2}
is easily shown to be equivalent to the first equation in 
Corollary 11
in 
\cite{Wang13e}.
Note that we derive BT1 in the primal solution space,
while Wang et al. \cite{Wang13e}
derived the identical test in the dual space.

%
% AppD: epsilon-path
%
\section{$\eps$-approximation Path Procedure}
\label{app:eps-path}
The
$\eps$-path algorithm enables us to compute an SVM regularization path
such that 
the relative approximation error between two consecutive solutions are bounded
by a small constant
$\eps$
(we set $\eps = 10^{-3}$). 
Precisely speaking,
the sequence of the regularization parameters 
$\{C_t\}_{t \in \NN_T}$
produced by
the $\eps$-path algorithm
has a property that, 
for any 
$C_{t - 1}$
and
$C_{t}$, 
$t \in \{2, \ldots, T\}$,
the former dual optimal solution 
$\bm \alpha^*_{C_{t - 1}}$
satisfies 
\begin{eqnarray}
\label{eq:relative.epsilon.approximation}
\frac{
 | \cD(\bm \alpha_{[C]}^*)
 - \cD( {\scriptstyle\frac{C}{C_{t-1}}} \bm \alpha_{[C_{t - 1}]}^*)|
 }{
 \cD(\bm \alpha_{[C]}^*)
 } \le \eps
 ~
 \forall
 ~
 C \in [C_{t - 1}, C_t],
\end{eqnarray}
where
$\cD$
is the dual objective function defined in
\eq{eq:d.SVM.prob.}.
This property roughly implies that, 
the optimal solution 
$\bm \alpha^*_{[C_{t - 1}]}$
is a reasonably good approximate solutions
within the range of 
$C \in [C_{t - 1}, C_{t}]$.

Algorithm \ref{alg:eps.app.path}
describes the regularization path computation procedure 
with
the safe sample screening and the $\eps$-path algorithms.
Given
$\bm w^*_{[C_{t - 1}]}$, 
the $\eps$-path algorithm finds the largest $C_t$
such that any solutions
between
$[C_{t - 1}, C_t]$
can be approximated by 
the current solution
in the sense of 
\eq{eq:relative.epsilon.approximation}. 
Then,
the safe sample screening rules for 
$\bm w^*_{[C_t]}$
are constructed by using
$\bm w^*_{[C_{t - 1}]}$
as the reference solution.
After screening out a subset of the samples, 
an SVM solver
({\libsvm} and {\liblinear} in our experiments)
is applied to the reduced set of the samples
to obtain 
$\bm w^*_{[C_t]}$.

% -------------------------
% alg1
% -------------------------
\begin{algorithm}[h]
\caption{SVM regularization path computation
with the safe sample screening and the $\eps$-path algorithms} 
\label{alg:eps.app.path}
\begin{algorithmic}[1]
\Require{Training set $\{(\bm x_i,y_i)\}_{i\in\bN_n}$}, the largest regularization parameter $C_T$. 
\Ensure{Regularization path $\{\bm w_{[C_t]}^*\}_{t \in \NN_T}$}.
%\Statex
\State Compute $C_{\rm min}$.
\State $t \leftarrow 1,~C_t \leftarrow C_{\rm min},~\bm \alpha_{[C_t]}^* \leftarrow C_{t} \ones$. 
\While {$\cL \neq \emptyset$ and $C_t < C_T$}

 \State $t \leftarrow t+1$.

 \State Compute the next $C_t$ by the $\eps$-path algorithm. 

 \State Construct the safe rules for $C_t$ by using $\bm w^*_{[C_{t - 1}]}$.

 \State Screen out a subset of the samples by those rules.
 
 \State Compute 
 $\bm w^*_{[C_t]}$
 by an SVM solver.

 \EndWhile
\end{algorithmic}
\end{algorithm}

\ifCLASSOPTIONcaptionsoff
  \newpage
\fi

%\clearpage
\bibliographystyle{IEEEtran}
\bibliography{eref}

\end{document}